\documentclass{article}
\usepackage{iclr2026_conference,times}

\usepackage{amsmath,amsfonts,bm}

\def\1{\bm{1}}

\DeclareMathAlphabet{\mathsfit}{\encodingdefault}{\sfdefault}{m}{sl}
\SetMathAlphabet{\mathsfit}{bold}{\encodingdefault}{\sfdefault}{bx}{n}

\usepackage{amsmath}
\usepackage{amssymb}
\usepackage{amsthm}
\usepackage{thmtools}
\usepackage{thm-restate}
\usepackage{booktabs}
\usepackage{array}
\usepackage{graphicx}
\usepackage{microtype}
\usepackage{placeins}
\usepackage{float}
\usepackage{subcaption}
\usepackage{url}
\usepackage{hyperref}
\hypersetup{hidelinks}
\usepackage{wrapfig}
\usepackage{algorithm}
\usepackage{algorithmic}
\usepackage{titletoc}
\titlecontents{section}[1.6em] {\smallskip} {\contentslabel{1.6em}} {\hspace*{-1.6em}} {\titlerule*[0.6pc]{.}\contentspage} \titlecontents{subsection}[3.9em] {} {\contentslabel{2.5em}} {\hspace*{-2.5em}} {\titlerule*[0.6pc]{.}\contentspage}

\makeatletter
\renewcommand{\paragraph}{\@startsection{paragraph}{4}{\z@}{0.45ex plus 0.3ex minus .2ex}{-1em}{\normalsize\bf}}
\def\thm@space@setup{\thm@preskip=2pt plus 1pt minus 1pt \thm@postskip=\thm@preskip} \makeatother

\AtBeginDocument{%
  \setlength{\abovedisplayskip}{4pt plus 2pt minus 2pt}%
  \setlength{\belowdisplayskip}{4pt plus 2pt minus 2pt}%
  \setlength{\abovedisplayshortskip}{3pt plus 2pt minus 2pt}%
  \setlength{\belowdisplayshortskip}{4pt plus 2pt minus 2pt}%
} 

\newcommand{\fitwidth}[1]{%
  \resizebox{\ifdim\width>\linewidth\linewidth\else\width\fi}{!}{#1}%
}

\newtheorem{theorem}{Theorem} \newtheorem{lemma}{Lemma} \newtheorem{assumption}{Assumption} \newtheorem{definition}{Definition}  \newtheorem{proposition}{Proposition}  \newtheorem{remark}{Remark}

\makeatletter
\newcommand{\thmspacingfloor}{%
  \def\th@plain{\thm@preskip 6pt plus 1pt minus 1pt \thm@postskip 5pt plus 1pt minus 1pt \itshape}%
  \def\th@definition{\thm@preskip 6pt plus 1pt minus 1pt \thm@postskip 5pt plus 1pt minus 1pt \normalfont}%
  \theoremstyle{plain}}
\makeatother

\usepackage{tcolorbox}
\tcbset{ colback=gray!10!white,        
    colframe=gray!10!white,       
    boxrule=0mm,                  
    sharp corners,                
    width=0.48\textwidth,             
    before=\vspace{-1mm},         
    after=\vspace{-1mm},          
    boxsep=0.5mm,                 
    left=1mm,                     
    right=1mm,                    
    top=0.5mm,                    
    bottom=0.5mm                  
}

\title{DSReg: Provably Recovering Individual World Latents without Reconstruction}

\author{
  Yujia Zheng\textsuperscript{1}
  David Klindt\textsuperscript{2}
  Randall Balestriero\textsuperscript{3}
  Bernhard Sch\"olkopf\textsuperscript{4,5}\\
  \textsuperscript{1}University of Illinois Urbana-Champaign
  \textsuperscript{2}Cold Spring Harbor Laboratory\\
  \textsuperscript{3}Brown University
  \textsuperscript{4}Max Planck Institute for Intelligent Systems
  \textsuperscript{5}ELLIS Institute T\"ubingen\\
\normalfont Project page: \url{https://dsreg.github.io/}
}

\iclrfinalcopy
\begin{document}
\maketitle

\begin{abstract}
\looseness=-1
Methods that recover individual latent variables of the world, from nonlinear ICA to dictionary learning and causal representation learning, anchor the latents to observations through reconstruction, auxiliary supervision, or distributional asymmetries such as non-Gaussianity. Methods without these anchors, including joint-embedding predictive architectures (JEPAs), identify the latent state only up to a linear transformation, so individual latents remain mixed. We close this gap: individual world latents can be provably recovered with no reconstruction, no decoder, and no labels. The key condition is \emph{Structural Diversity}: different latents leave distinct dependency footprints on observations, just as \textit{no two snowflakes are alike}. Building on the linear identifiability that LeJEPA provides, we prove that under Structural Diversity, \textbf{\emph{DSReg}} (\textbf{\emph{D}}ependency-\textbf{\emph{S}}parsity \textbf{\emph{Reg}}ularization) recovers individual world latents up to signed permutation, without reconstruction or a decoder. It applies post hoc to any linearly identified representation, reusing trained checkpoints at no loss over joint training, and establishes the first fully identifiable JEPA that recovers every world latent. Moreover, as a condition on dependency footprints, Structural Diversity is strictly weaker than all structural conditions of prior identifiable latent variable models. Across synthetic regimes, world model probes, learned visual encoders, and external renderers, DSReg preserves dense prediction while improving individual-latent recovery and downstream use with scales.
\end{abstract}

\vspace{-0.2em}
\section{Introduction}
\vspace{-0.2em}

World models are useful when their learned variables can be treated as variables of the world. A representation may contain the full latent state and still hide every physical factor inside a dense mixture of learned variables. Such mixing can leave predictive performance untouched while taking away the variables needed for intervention, precise control, visual editing, and mechanistic inspection. Many modules built on top of a world model act through a few variables at a time, such as a controller moving one object or a monitor watching one factor, and for such uses a well-spanned mixture may not be enough. A model that represents a scene through position, velocity, contact, and object attributes should isolate those individual latents rather than fold them into one dense state vector.

\begin{figure}[b!]
\centering
\vspace{-2pt}
\includegraphics[width=0.8\linewidth]{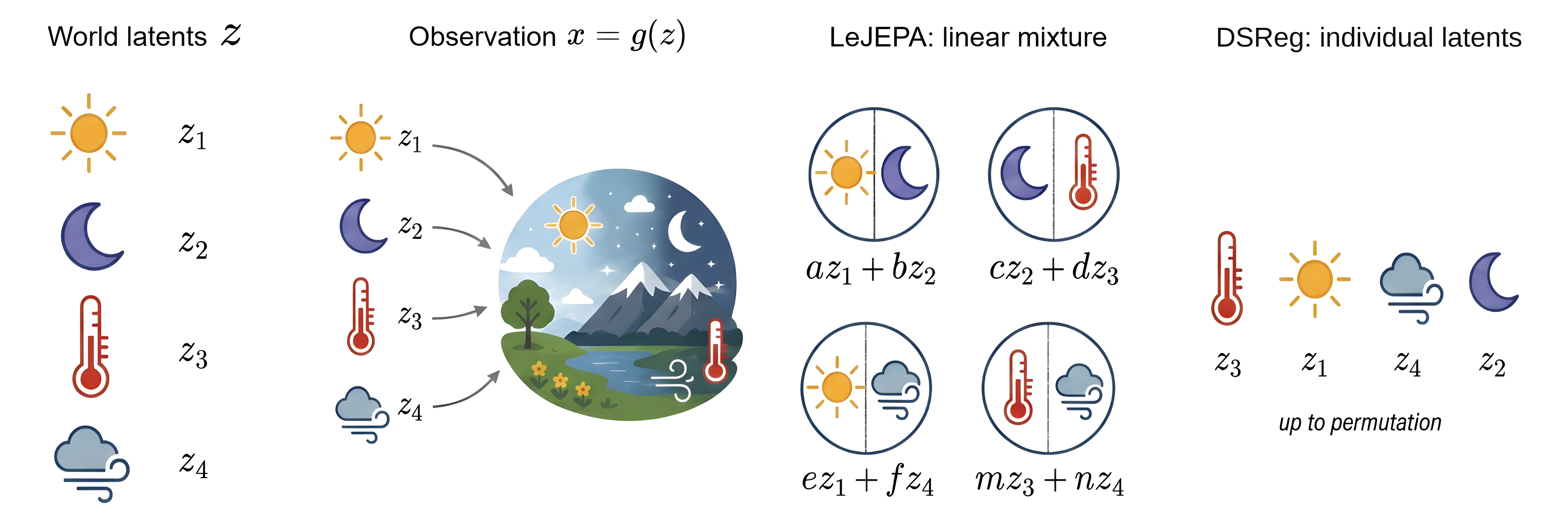}
\caption{\textbf{DSReg turns a mixed latent state into individual latents.} World latents $z$ generate observations $x=g(z)$. LeJEPA identifies the latent state up to a linear transformation, $\hat z = A z$, so each learned variable can remain a mixture of ground-truth latents. DSReg recovers individual latents up to signed permutation.}
\label{fig:intro}
\vspace{-8pt}
\end{figure}

Joint-embedding predictive architectures (JEPAs) learn world models by predicting in representation space instead of reconstructing observations \citep{lecun2022path}. The choice, however, removes the anchor that grounds latents elsewhere in the identifiability literature. No decoder, observation likelihood, or auxiliary signal ties each learned variable to a variable of the world. Collapse, the best-known symptom of this missing anchor, is handled by an established toolkit \citep{grill2020bootstrap,chen2021exploring,zbontar2021barlow,bardes2022vicreg}. At the same time, avoiding collapse does not recover the latents of the world: a non-collapsed representation may still encode them as an invertible mixture. What is needed is identifiability, the guarantee that every representation consistent with the training criterion matches the true latents up to a known ambiguity, such as a rotation.

Theory has begun to catch up. Under Gaussian latent-world assumptions, LeJEPA identifies the latent state up to a linear transformation \citep{klindt2026worldmodel}, separating the representations that merely predict well from those that align with the hidden generative process of the world. Yet each individual latent can still be a mixture of several ground-truth factors (Figure~\ref{fig:intro}), and the mixture does harm where world models are expected to help. Consider a household robot asked to slide a cup to the left while keeping a nearby knife fixed: if one learned latent mixes the cup position with the knife position, changing that latent is no longer a targeted cup movement, and a correction meant for the cup can also shift the knife. Thus, a fundamental question about learning world models remains open:

\begin{center}
\vspace{-3pt}
\emph{Can we learn individual world latents without reconstruction?}
\vspace{-5pt}
\end{center}

The question has deep roots in the identifiability literature. In nonlinear ICA, one line of work uses auxiliary variables as weak supervision \citep{hyvarinen2016unsupervised,hyvarinen2019nonlinear,yao2021learning,halva2021disentangling,lachapelle2021disentanglement}, while another constrains the mixing function itself \citep{taleb1999source,moran2022sparsevae,kivva2022identifiability,zheng2022nonlinearica,buchholz2022function}. A reconstruction-free line reaches component-wise recovery through contrastive objectives, but leans on nonstationarity, non-Gaussian temporal dependence, or a non-Gaussian conditional \citep{hyvarinen2016unsupervised,hyvarinen2017pcl,zimmermann2021contrastive}, none of which a stationary Gaussian world state offers (Appendix~\ref{app:related}). Meanwhile, in causal representation learning, identifiability typically rests on interventional data \citep{von2023nonparametric,jiang2023learning,jin2023learning,zhang2024causal} or counterfactual views \citep{von2021self,brehmer2022weakly}. Each of these routes anchors the latents in reconstruction, an explicit likelihood, or symmetry-breaking side information. In contrast, methods that discard every anchor stop at the linear guarantee above, and recovering individual latents without any of these anchors has remained an open problem.

\paragraph{Contributions.}\looseness=-1 We close this gap. We prove, to the best of our knowledge, the first component-wise identifiability result that requires no reconstruction, labels, or distributional asymmetry. Under \emph{Structural Diversity}, which asks only that different latents leave distinct dependency footprints on observations, together with a standard faithfulness condition, each latent variable of the world is identified up to signed permutation. Therefore, each learned variable is provably a variable of the hidden world. Moreover, we prove that Structural Diversity is strictly weaker than prior structural conditions across the identifiability literature. Our theory further prescribes an actionable objective, named \textbf{\emph{DSReg}} (\textbf{\emph{D}}ependency-\textbf{\emph{S}}parsity \textbf{\emph{Reg}}ularization): it applies post hoc to any linearly identified representation, reusing trained checkpoints at no loss relative to joint training, all without reconstruction or a decoder. Instantiated on a JEPA whose predictive objective supplies the linear-identifiability premise, DSReg resolves the Gaussian rotational invariance and pins down individual latents. We validate the theory and the method across synthetic regimes, world model probes, learned visual encoders, and external renderers: DSReg preserves dense prediction while improving individual-latent recovery and its downstream use, from sparse control and visual editing to few-shot latent readout of world state.

\vspace{-0.2em}
\section{Preliminaries}
\vspace{-0.2em}

\label{sec:preliminaries}
Identifiability asks when the latent variables that generated the observations can be recovered, up to relabeling and sign. Throughout, $z=(z_1,\ldots,z_d)$ is the latent state of the world, $x=g(z)$ an observation, and $h(z)=f_\theta(g(z))$ the estimated latents of a JEPA encoder $f_\theta$. The recovery we target is component-wise. Each estimated variable should be one variable of the world rather than an arbitrary unknown mixture of several of them, as the next definition makes precise.

\begin{definition}[Identifiability of individual world latents]
\label{def:identifiability}
A representation $\hat z=h(z)$ \emph{recovers the individual world latents}, and is called \emph{signed-permutation identifiable}, if $\hat z_i=s_iz_{\pi(i)}$ almost surely for every $i\in[d]$, for some permutation $\pi$ of $[d]$ and signs $s_1,\ldots,s_d\in\{\pm1\}$.
\end{definition}

\paragraph{Identifiability of latent variable models.} Nonlinear Independent Component Analysis obtains component-wise guarantees by adding structure beyond the marginal distribution of $x$: auxiliary variables \citep{hyvarinen2016unsupervised,hyvarinen2019nonlinear,yao2021learning,halva2021disentangling,lachapelle2021disentanglement,hyvarinen2024identifiability} or constraints on the mixing function \citep{taleb1999source,moran2022sparsevae,kivva2022identifiability,zheng2022nonlinearica,buchholz2022function}. Causal representation learning extends such guarantees to dependent latents through interventions \citep{von2023nonparametric,jiang2023learning,jin2023learning,varici2025score}, counterfactual or multi-view observations \citep{von2021self,brehmer2022weakly,yao2024multi,morioka2024causal}, distributional shifts \citep{zhang2024causal,ng2025causal}, or combinations of complementary constraints \citep{li2025synergy,reizinger2025identifiable,yao2025unifying}. Meanwhile, parallel results exist for factor analysis \citep{hu2008identification} and dictionary learning \citep{zheng2026diverse}. Nearly all of these strategies, however, anchor the latents through observation reconstruction or an explicit likelihood. That anchor is exactly what the JEPA objective removes, so the existing techniques do not transfer to settings without reconstruction.

\paragraph{JEPAs and linear identifiability.} A JEPA learns by predicting the embedding of a target observation from the embedding of a related context observation, rather than reconstructing the observation itself. In a world model the pairs are consecutive states of a temporal process, $x_t=g(z_t)$ and $x_{t+1}=g(z_{t+1})$; as in \citet{klindt2026worldmodel} we assume the process is stationary, so all time slices share one law and we drop the subscript, writing $(x,x^+)$. Given such a pair, an encoder $f_\theta$ is trained so that the representation of $x$ aligns with that of $x^+$, while constraints prevent collapse. The objective is therefore deliberately non-reconstructive: no decoder maps representations back to observations, and no reconstruction error or observation likelihood enters training. What it can identify on its own is the weaker notion of linear identifiability.

\begin{definition}[Linear identifiability]
\label{def:linear-identifiability}
A representation $\hat z=h(z)$ is \emph{linearly identifiable} if $\hat z=Az$ almost surely for some invertible $A\in\mathbb R^{d\times d}$.
\end{definition}

LeJEPA \citep{balestriero2025lejepa} instantiates this design with SIGReg, a scalable regularizer that drives the embedding toward an isotropic Gaussian. In a Gaussian latent dynamical system, \citet{klindt2026worldmodel} proves the following characterization of the ambiguity DSReg must resolve.

\begin{theorem}[LeJEPA linear identifiability \citep{klindt2026worldmodel}]
\label{thm:lejepa-prelim}
Let $z\sim\mathcal{N}(0,I_d)$, and let the target latent follow the stationary Gaussian transition $z^+=\rho z+\sqrt{1-\rho^2}\,\eta$ with $\eta\sim\mathcal{N}(0,I_d)$ independent of $z$, and $\rho\in(0,1)$. Let $x=g(z)$, $x^+=g(z^+)$, and let $h=f_\theta\circ g:\mathbb{R}^d\to\mathbb{R}^d$ be measurable with $h(z)\sim\mathcal{N}(0,I_d)$. Then the population alignment loss satisfies
\begin{equation}
    \mathcal{L}_{\mathrm{align}}(h) = \mathbb{E}\big[\|h(z^+)-h(z)\|_2^2\big] \;\ge\; 2(1-\rho)d,
\end{equation}
with equality if and only if $h(z)=Qz$ almost surely for some orthogonal $Q\in O(d)$. At the alignment optimum, $h$ is linearly identifiable in the sense of Definition~\ref{def:linear-identifiability}, with orthogonal $A=Q$.
\end{theorem}

Under LeJEPA's Gaussian-world conditions, alignment plus the Gaussian constraint identifies the latent process up to a rotation. Each estimated variable may still be a mixture of true latents, and the next section explains why \textit{no distributional criterion can ever do better}.

\vspace{-0.2em}
\section{Identifiability from Structural Diversity}
\vspace{-0.2em}

\label{sec:theory}
The theory develops in three steps: Section~\ref{sec:linear-to-signed} isolates the obstacle, a distributionally undetectable rotation that can mix every latent; Section~\ref{sec:structural-diversity} builds the dependency structure that this rotation cannot hide from and proves the main identifiability theorem; Section~\ref{sec:practical-estimator} turns DSReg into a practical estimator whose recovery guarantee survives estimation error in the Jacobians.

\vspace{-0.1em}
\subsection{From linear to signed-permutation identifiability}
\label{sec:linear-to-signed}
\vspace{-0.1em}

\paragraph{Linear is not component-wise.} Linear identifiability does not imply recovery of individual latents. The counterexample is any rotation. If $\hat z=Az$ is linearly identifiable and $U\in O(d)$ is not a signed permutation, then $U\hat z=(UA)z$ satisfies Definition~\ref{def:linear-identifiability} equally well, yet each of its variables is a mixture of several world latents. The two notions are separated by exactly this freedom. In the Gaussian JEPA setting no per-latent rescaling enters the comparison, because both $z$ and the representation are standardized and the remaining ambiguity from Theorem~\ref{thm:lejepa-prelim} is orthogonal.

\paragraph{The gap matters.} A representation whose variables mix object position with illumination predicts the scene as well as one that separates them, but a controller that is allowed to change only one latent cannot move the object without also changing the lighting. Under the premise $h(z)=Qz$, a generic row of $Qz$ mixes all $d$ latents at once, so a linearly correct representation may still be ill suited to the downstream tasks that motivate learning a world model in the first place.

\paragraph{The obstacle: rotation invariance.} The linear indeterminacy is not unique to LeJEPA. It limits any objective that evaluates representations only through their distribution, because the Gaussian distribution is rotation invariant. An isotropic Gaussian is like a ball. Its density $p(z)\propto\exp(-\|z\|^2/2)$ depends on $z$ only through its norm, so rotating changes nothing and $Uz \sim \mathcal{N}(0, I_d)$ for every $U \in O(d)$. A non-Gaussian law such as Laplace would break this symmetry, which is how contrastive nonlinear ICA separates components. The Gaussian state gives it nothing to work with. Thus, $h$ and $Uh$ are distributionally indistinguishable, even though the rotation $U$ may arbitrarily mix the latent variables. An unknown rotation breaks any correspondence $h$ already achieves, and downstream use can then silently destroy the alignment without any warning signal that it happened.


\paragraph{From obstacle to signal.} The same observation says where a fix must come from. If no function of the latent distribution can detect the rotation, the identifying signal must live outside that distribution, in something observable that does change when the latents are rotated. Fortunately, we find such a signal in how observations depend on the estimated latents, as we elaborate in the next section.

\begin{tcolorbox}[width=\linewidth]
\textbf{Insight.}\ Rotations of a Gaussian latent state are undetectable to every distributional criterion, yet they mix the variables. Recovering individual latents needs a rotation-sensitive signal. DSReg takes that signal from how the observations depend on the estimated latents.
\end{tcolorbox}

\vspace{-0.15em}
\subsection{Dependency supports and Structural Diversity}
\label{sec:structural-diversity}
\vspace{-0.1em}

\looseness=-1 DSReg takes the rotation-sensitive signal from how observed variables depend on latents. Let $x=g(z)\in\mathbb R^p$ and define the dependency Jacobian $D(z)=\partial x/\partial z$. Throughout, $\mu$ is the law of $z$ (standard Gaussian under Theorem~\ref{thm:lejepa-prelim}). The map $g$ is differentiable $\mu$-almost everywhere with measurable Jacobian, so $D(z)$ is defined $\mu$-a.e. The \emph{dependency support} declares an entry active when its partial derivative is nonzero on a set of positive $\mu$ measure. For matrix-valued $M(z)$ we count active entries by
\[
    \|M(\cdot)\|_{0,\mu} = \left| \{(r,i):M_{ri}(z)\ne0 \text{ on a set of positive }\mu\text{ measure}\} \right|,
\]
a count with values in $\{0,1,\ldots,pd\}$. Only finitely many values occur, so some rotation attains the smallest value and minimizers over $O(d)$ exist. Neither $g$ nor $D$ is assumed known. Section~\ref{sec:practical-estimator} estimates the dependency structure from observations and estimated latents alone.

\paragraph{What rotation disturbs.} The support is what rotation changes while the latent distribution stays fixed. If one latent controls position and affects one set of observed variables while another controls color and affects a different set, the rotation $(z_i\pm z_j)/\sqrt2$ leaves the Gaussian law unchanged but makes each rotated latent act on the union of the two sets. The Jacobian support records this mixing. 

\begin{wrapfigure}[16]{r}{0.23\textwidth}
    \vspace{-10pt}
    \centering
    \includegraphics[width=0.90\linewidth]{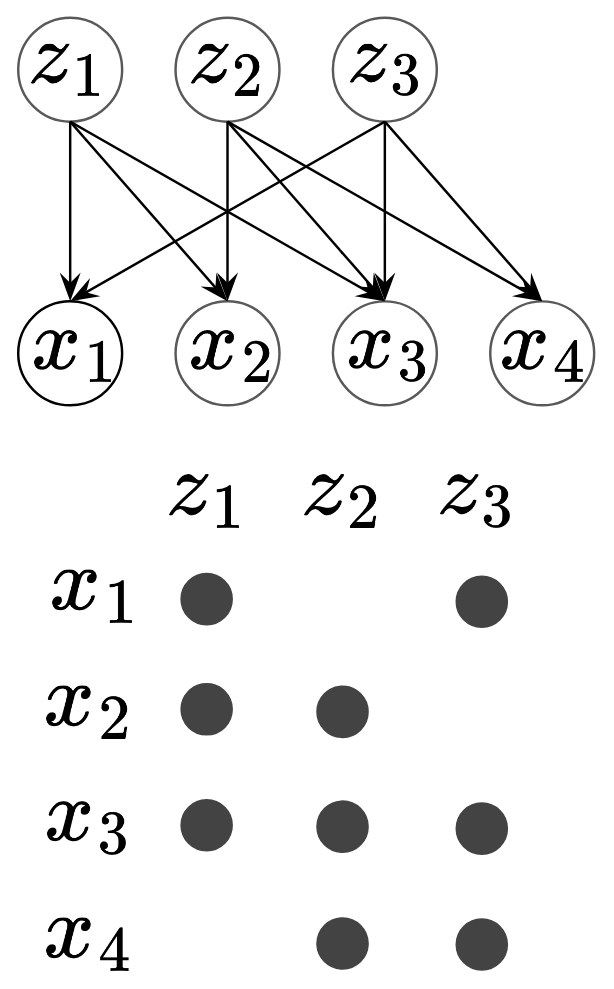}
    \vspace{5pt}
    \captionsetup{justification=raggedright,singlelinecheck=false,skip=2pt} \caption{\textbf{Dependency footprints.} Each column of $\partial x/\partial z$ marks which observed variables one latent affects.}
    \label{fig:dependency-footprints}
\end{wrapfigure}

The same example shows what structure must supply. Define the \emph{dependency footprint} of latent $z_i$ as
\[
    \mathcal S_i=\{\,r:D_{ri}(z)\ne0 \text{ on a set of positive }\mu\text{ measure}\,\},
\]
the set of observed variables that $z_i$ affects. Figure~\ref{fig:dependency-footprints} shows the same object as a support matrix. If two latents affected exactly the same observed variables, structure alone could not distinguish them, since a rotation confined to the pair changes nothing dependency structure can see. Identifiability from structure therefore requires different latents to leave different dependency footprints on the observed variables.

\begin{assumption}[Structural Diversity]
\label{asm:structural-diversity}
The dependency footprints are pairwise distinct: $\mathcal S_i\ne\mathcal S_j$ for all $i\ne j$.
\end{assumption}

\paragraph{A structural condition.} Structural Diversity constrains the dependency structure alone and places no requirement on the latent distribution. It asks only that no two latents affect exactly the same set of observed variables. Footprints may overlap, differ in a single variable, or even be nested.

To connect the dependency structure to the residual ambiguity, recall that training leaves one unknown rotation: $h(z)=Qz$ with $Q\in O(d)$. DSReg searches over a further rotation $R\in O(d)$, producing candidate latents $\tilde z=Rh(z)$. The composite map from true to candidate latents is $RQ$, and the goal is to choose $R$ so that $RQ$ is a signed permutation. The next lemma computes how the dependency Jacobian transforms under this search.

\begin{restatable}[Dependency bridge]{lemma}{dependencybridge}
\label{lem:dependency-bridge}
Suppose a representation is linearly identifiable with orthogonal ambiguity, so that $h(z)=Qz$ for some $Q\in O(d)$, as Theorem~\ref{thm:lejepa-prelim} provides for JEPA representations, and let $x=g(z)$ have dependency Jacobian $D(z)=\partial x/\partial z$ almost everywhere. Then for candidate latents $\tilde z=Rh(z)$ with $R\in O(d)$,
\begin{equation}
    \frac{\partial x}{\partial \tilde z} = D(z)Q^\top R^\top \qquad\text{almost everywhere}.
    \label{eq:theory-chain-rule}
\end{equation}
\end{restatable}
The identity is nothing but the chain rule, since the premise makes the map between candidate and true latents the linear bijection $z=Q^\top R^\top\tilde z$. It is worth noting that no invertibility of the encoder on observations is needed, in keeping with an encoder that discards low-level detail (Appendix~\ref{app:proofs}). The ambiguity thus enters the dependency structure as right multiplication by the orthogonal $U=Q^\top R^\top$, and DSReg selects the remaining rotation by minimizing the dependency support. Among the representations in the identified orthogonal class, it prefers the one whose latents have the sparsest dependency on the observations,
\begin{equation}
    \min_{R\in O(d)} \Big\|\frac{\partial x}{\partial\tilde z}\Big\|_{0,\mu}, \qquad \tilde z=Rh(z).
    \label{eq:idjepa-population-criterion}
\end{equation}
\paragraph{Why the support sees the rotation.} For a Givens rotation on a fixed generic matrix, the possible effects on the support form a short taxonomy \citep{ghassami2020characterizing}. Apart from column swaps, any rotation that mixes two columns deactivates the targeted entry and activates entries wherever the two column supports differ, so it always leaves a visible trace (Proposition~\ref{prop:support-rotations} in Appendix~\ref{app:proofs} restates the taxonomy in our setting). For the function-valued dependency Jacobian, the support of a mixed candidate latent follows the union-support formula of Lemma~\ref{lem:functional-union-support} (Appendix~\ref{app:proofs}), under Assumption~\ref{asm:functional-no-cancellation} below. Under Structural Diversity, a rotation mixing latents must change the support.

\begin{assumption}[Functional no-cancellation]
\label{asm:functional-no-cancellation}
For each observation row $r$, write
\[
    I_r=\{i:D_{ri}(z)\ne0\text{ on a set of positive }\mu\text{ measure}\}.
\]
Then any constant-coefficient relation among the active derivative functions is trivial:
\[
    \sum_{i\in I_r} c_iD_{ri}(z)=0 \quad \text{for }\mu\text{-a.e. }z \quad\Longrightarrow\quad c_i=0 \text{ for all } i\in I_r .
\]
\end{assumption}

\begin{restatable}[Component-wise identifiability without reconstruction]{theorem}{idjepazheng}
\label{thm:idjepa-zheng}
In the setting of Lemma~\ref{lem:dependency-bridge}, assume Structural Diversity (Assumption~\ref{asm:structural-diversity}) and Functional no-cancellation (Assumption~\ref{asm:functional-no-cancellation}). Then for every minimizer $R$ of the support-sparsity criterion~\eqref{eq:idjepa-population-criterion}, the matrix $RQ$ is a signed permutation: there exist a permutation $\pi$ of $[d]$ and signs $s_1,\ldots,s_d\in\{\pm1\}$ such that the candidate latents $\tilde z=Rh(z)$ satisfy $\tilde z_i=s_iz_{\pi(i)}$ almost surely for every $i\in[d]$.
\end{restatable}

After the DSReg rotation, each estimated latent is therefore one individual world latent up to sign and relabeling, the only ambiguity left, and one that keeps each variable individually meaningful.

\begin{tcolorbox}[width=\linewidth]
\textbf{Insight.}\ Theorem \ref{thm:idjepa-zheng} takes reconstruction-free world models beyond the linear indeterminacy. A representation that spans the world state becomes one whose variables are the world's individual latents up to signed permutation, without decoder, likelihood, or labels. This is the difference between holding all the latent information and being able to reach any one piece of it.
\end{tcolorbox}

Functional no-cancellation is a faithfulness-type genericity requirement. It rules out constant-coefficient cancellations that would hide genuine Jacobian edges. It plays the same role as sufficient nonlinearity and sufficient variability in nonlinear identifiability, where Jacobian variation across samples must span each active support \citep{hyvarinen2016unsupervised,khemakhem2020ivae,sorrenson2020disentanglement,lachapelle2021disentanglement,zheng2022nonlinearica,kong2023identification,yan2023counterfactual,zhang2024causal,lachapelle2024additive}. Moreover, within smooth observation families with fixed supports and analytic Gram determinants, and provided each row's determinant is nonzero somewhere in the family, the condition holds for Lebesgue-almost every parameter value (Proposition~\ref{prop:no-cancel-generic}).

\paragraph{Strictly weaker than prior conditions.} Structure-based identifiability rests on sparsity conditions on the mixing Jacobian in nonlinear ICA \citep{zheng2022nonlinearica} and on column supports in linear ICA \citep{ng2023sparseica}. Proposition~\ref{prop:diversity-weaker} states these conditions directly on the footprints and orders them. Every one forces pairwise-distinct footprints, and no converse implication holds.

\begin{restatable}[Structural Diversity is strictly weaker]{proposition}{diversityweaker}
\label{prop:diversity-weaker}
Stated on the dependency footprints $\mathcal S_1,\ldots,\mathcal S_d$, each of the following conditions implies Structural Diversity (Assumption~\ref{asm:structural-diversity}), and none of the converse implications holds:
\begin{enumerate}
\vspace{-0.2em}
\item \emph{Structural Sparsity, intersection form} \citep{zheng2022nonlinearica}: for every $k\in[d]$ there is a nonempty set $\mathcal C_k$ of observed variables with $\bigcap_{r\in\mathcal C_k}\{\,i\in[d]: r\in\mathcal S_i\,\}=\{k\}$;
\vspace{-0.2em}
\item \emph{Structural Sparsity, overlap-rank form} \citep{zheng2022nonlinearica}: for every $\mathcal C\subseteq[d]$ with $|\mathcal C|\ge2$ and every $k\in\mathcal C$, $\big|\bigcup_{j\in\mathcal C}\mathcal S_j\big|-\operatorname{rank}(M^{\mathcal C})>|\mathcal S_k|$, where $M^{\mathcal C}$ is the binary matrix with $M^{\mathcal C}_{rj}=1$ exactly when $j\in\mathcal C$ and $r$ lies in $\mathcal S_j$ and in $\mathcal S_{j'}$ for some other $j'\in\mathcal C$;
\vspace{-0.2em}
\item \emph{Non-Inclusion}, a common weakening of both forms above: $\mathcal S_i\not\subseteq\mathcal S_j$ for all $i\ne j$;
\vspace{-0.2em}
\item \emph{Structural Variability} \citep{ng2023sparseica}: $|\mathcal S_i\,\triangle\,\mathcal S_j|\ge2$ for all $i\ne j$.
\end{enumerate}
\vspace{-0.2em}
\end{restatable}

\paragraph{The difference matters in practice.} A global latent that touches every observed variable (illumination, camera gain, a global style factor) has a footprint that strictly contains those of localized latents. Nested footprints of this kind violate Non-Inclusion, and hence both forms of Structural Sparsity, yet satisfy Structural Diversity, so Theorem~\ref{thm:idjepa-zheng} still applies. The boundary is equally sharp in the other direction. If $\mathcal S_i=\mathcal S_j$, every rotation in the $(i,j)$ plane leaves both rotated columns active on the same footprint, the support count acquires non-permutation minimizers, and distinguishing the pair would then require information beyond the dependency support pattern itself.

\begin{tcolorbox}[width=\linewidth]
\textbf{Insight.}\ Prior structural conditions were built for reconstruction-based models, where a decoder anchors the latents. Proposition~\ref{prop:diversity-weaker} brings structure-based identifiability to the reconstruction-free setting while even weakening the condition, suggesting that Structural Diversity is the weakest among all structural conditions, possibly even a necessary one for recovering individual latents.
\end{tcolorbox}

\vspace{-0.1em}
\subsection{The DSReg estimator}
\label{sec:practical-estimator}
\vspace{-0.1em}

\looseness=-1 The full method, LeJEPA + DSReg, has two objectives: predictive Gaussian learning identifies the orthogonal solution class, and DSReg selects the sparsest representative within that class. This decomposition is exact at the population level. For every $R\in O(d)$, alignment obeys $\mathcal L_{\mathrm{align}}(Rh)=\mathcal L_{\mathrm{align}}(h)$ and the Gaussian constraint is preserved. The orthogonal DSReg head may therefore be optimized alongside predictive learning or after it. We recommend the decoupled schedule, which is simpler, cheaper, and reuses existing checkpoints, and an ablation confirms the convenience costs nothing. The schedules recover equally well, with the joint head marginally lower until the decoupled fit completes it (Appendix Table~\ref{tab:joint-ablation}). Specifically, on the whitened representation, DSReg estimates local maps $B_a\approx \partial x/\partial h$ by ridge regression at $m$ anchors and minimizes the anchor-averaged $\ell_1$ relaxation of the population support criterion~\eqref{eq:idjepa-population-criterion}:
\begin{equation}
    \mathcal L_{\mathrm{DSReg}}(R) =\frac{1}{m}\sum_{a=1}^{m}\|B_aR^\top\|_1, \qquad R\in O(d) .
    \label{eq:idjepa-objective}
\end{equation}
\looseness=-1 Here $B_aR^\top$ estimates how observations depend on the candidate latents $Rh$. The local regressions read the observations at the anchors, as any identifying signal ultimately must, but they provide no global observation model, train no decoder, and never optimize reconstruction. Theorems~\ref{thm:idjepa-zheng} and~\ref{thm:approx-recovery} analyze the support criterion and its thresholded perturbation, while Equation~\eqref{eq:idjepa-objective} is the continuous finite-sample relaxation used in every experiment. Because the $\ell_1$ norm weighs magnitude as well as count, the relaxation and the support criterion can in principle select different rotations. On the benchmarks of Figure~\ref{fig:optimizer-ablation} an exact support count recovers at least as well as the $\ell_1$ default, a refinement of the $\ell_1$ solution by the exact count closes the remaining gap, and differentiability is what lets the objective scale to large $d$ and $p$ (Appendix~\ref{sec:app-scaling}). Ground-truth latents are never accessible during estimation and enter only through benchmark construction and evaluation.

\begin{algorithm}[h]
\footnotesize \caption{DSReg: recovering individual latents without reconstruction}
\label{alg:idjepa}
\begin{algorithmic}[1]
\STATE Train the predictive Gaussian objective to identify the orthogonal solution class; whiten and freeze the representation $h$.%
\STATE Choose the observed vector used for dependency fitting; sample $m$ anchor points $h_1,\ldots,h_m$ from the data.%
\STATE For each anchor, estimate $B_a\approx\partial x/\partial h$ by ridge regression of observations on $h$ over the $k$ nearest neighbors of $h_a$ in representation space.%
\STATE Select $R=\arg\min_{R\in O(d)}\mathcal L_{\mathrm{DSReg}}(R)$.%
\STATE \textbf{return} the DSReg representation $\tilde z=Rh$.
\end{algorithmic}
\vspace{1pt}
{\footnotesize\emph{Joint schedule:} the rotation may equally be trained alongside step~1, as a head updated during training that never feeds back into the encoder, with steps~2--4 as a final refit; recovery matches the decoupled schedule.}
\end{algorithm}

\paragraph{Scaling.}\looseness=-1 Equation~\eqref{eq:idjepa-objective} appears to need $m$ matrices of size $p\times d$, which would explode as soon as the observed dimension grows. Fortunately, it does not. Each $B_a$ is a ridge fit over the $k$ neighbors of its anchor, so it factors as $B_a=\Delta X_a^\top P_a$ with $P_a\in\mathbb R^{k\times d}$, and the observations $\Delta X_a$ are already in memory as the dataset. The criterion is therefore determined by factors costing $O(mkd)$ rather than $O(mpd)$, and can be evaluated without the large array ever existing. Appendix~\ref{sec:app-scaling} turns this one fact into three interchangeable ways to evaluate the same objective, materialized, exactly factored, and an unbiased row-sampled estimator whose per-step cost is independent of $p$ and $m$, and says which to use when. The full estimated procedure reaches $d=8192$ on one $48$\,GB GPU, with samples to $10^6$ and observed dimensions to $2^{17}$ leaving the per-step cost essentially flat.

\paragraph{Robustness to estimation error.} Whitening reduces any invertible linear ambiguity to an orthogonal one: $\operatorname{Cov}(h)^{-1/2}h(z)=\tilde Qz$ with $\tilde Q\in O(d)$ (Lemma~\ref{lem:whitening-reduction}, Appendix~\ref{app:robustness}). One may then wonder how estimated-Jacobian error affects the support target that DSReg relaxes.

Recovery degrades gracefully, governed by two constants with mechanical meanings. The first, $\sigma>0$, measures how strongly an active dependency registers in the data: the smallest $\lambda_{\min}(G_r)^{1/2}$ over the row Gram matrices $G_r=[\langle D_{ri},D_{ri'}\rangle_{L^2(\mu)}]_{i,i'\in I_r}$, positive under Assumption~\ref{asm:functional-no-cancellation}. The second, the pattern margin $\rho^*(\delta)>0$, measures how visibly a rotation $\delta$-far from every signed permutation must disturb the support pattern: the $(s^*{+}1)$-th largest of the sub-vector norms $\|U_{I_r,j}\|_2$ over such rotations, with $s^*=\sum_i|\mathcal S_i|$ the population minimum (Lemma~\ref{lem:pattern-margin}). A strong signal cannot be hidden by weak noise. Whenever the sum of the threshold and the noise level stays below $\sigma\rho^*(\delta)$, every minimizer lands within $\delta$ of a signed permutation.

\begin{restatable}[Approximate recovery under Jacobian perturbation]{theorem}{robusttheorem}
\label{thm:approx-recovery}
Let the assumptions of Theorem~\ref{thm:idjepa-zheng} hold with $d\ge2$, together with (essential supremum)
\[
\operatorname{ess\,sup}\nolimits_z\max\nolimits_r\|D_{r,:}(z)\|_2\le M<\infty .
\]
Let $\widehat D(z)=D(z)+E(z)$ with $E$ measurable and $\operatorname{ess\,sup}_z\max_r\|E_{r,:}(z)\|_2\le\varepsilon$, and for a threshold $\tau>\varepsilon$ define
\[
\widehat N_\tau(R) =\Big|\big\{(r,j):\mu\big(|(\widehat D(z)\,Q^\top R^\top)_{rj}|>\tau\big)>0\big\}\Big|, \qquad R\in O(d).
\]
If $\delta>0$ satisfies $\tau+\varepsilon<\sigma\,\rho^*(\delta)$, then minimizers of $\widehat N_\tau$ over $O(d)$ exist and every minimizer $\widehat R$ satisfies $\min_{P}\|\widehat RQ-P\|_F<\delta$ over signed permutations $P$.
\end{restatable}

Exact recovery is impossible under noise, since thresholding at any $\tau>0$ ties the count on rotations $O(\tau)$-close to a signed permutation, so $\delta$-closeness is the right notion. The statement concerns the thresholded population criterion. Appendix Figure~\ref{fig:structdiv-noise-robustness} stress-tests the $\ell_1$ estimator under noise.

\begin{tcolorbox}[width=\linewidth]
\textbf{Insight.}\ DSReg needs no retraining or decoder. Predictive learning fixes the orthogonal class and DSReg, fit jointly with it or afterwards, selects individual latents. Theorem~\ref{thm:approx-recovery} shows the thresholded population target is stable to Jacobian estimation error, and the $\ell_1$ implementation is exercised directly in ablations. Dense performance is untouched because a rotation loses no information.
\end{tcolorbox}%

\vspace{-0.2em}
\section{Experiments}
\vspace{-0.2em}

\label{sec:experiments}
The experiments validate the theory and support downstream applications. Encoders trained from scratch recover individual latents wherever Structural Diversity holds (Section~\ref{sec:exp-synthetic}). The regime map, failure included, is exactly the theory's (Section~\ref{sec:exp-regimes}). The recovered latents pay off in sparse control, editing, prediction, and monitoring (Section~\ref{sec:exp-sparse-use}). The gains survive encoders trained from pixels and renderers we did not design, reaching a supervised oracle (Section~\ref{sec:exp-visual-encoders}). Dense $R^2(h\to z)$, the best linear readout from estimated to ground-truth latents, measures whether the state is retained as a span; the mean correlation coefficient (MCC) after maximum-weight matching, the standard identifiability metric in nonlinear ICA \citep{hyvarinen2016unsupervised,khemakhem2020ivae}, measures one-to-one recovery; and sparse-use probes ask whether a module can act through a few estimated latents without a dense unmixing. Ground-truth latents define benchmark pairs and evaluation scores. DSReg is fit from estimated latents with raw observations or a fixed preprocessing chosen before fitting. Throughout, LeJEPA labels the frozen representation before the DSReg rotation, so each pair isolates what the rotation does to one representation. Below, $N$ denotes the benchmark's latent dimension, and results report mean~$\pm$~std over all runs with random seeds, five to fifteen for the learned visual encoders, and ten to twenty for the remaining families, with the exact count listed in every figure and table caption. Full setups and the supporting suite, led by the scaling study, are in Appendix~\ref{app:experiments}, which every claim in the subsections that follow points into.

\vspace{-0.1em}
\subsection{The full method recovers individual latents end to end}
\label{sec:exp-synthetic}
\vspace{-0.1em}

\begin{figure}[tbp]
    \centering
    \begin{subfigure}[b]{0.270\textwidth}
        \centering
        \includegraphics[width=\linewidth]{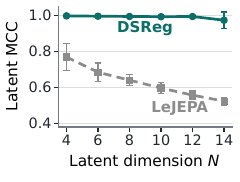}
        \caption{Learned encoders}
        \label{fig:synthetic-mcc}
    \end{subfigure}\hspace{0.03\textwidth}
    \begin{subfigure}[b]{0.610\textwidth}
        \centering
        \includegraphics[width=\linewidth]{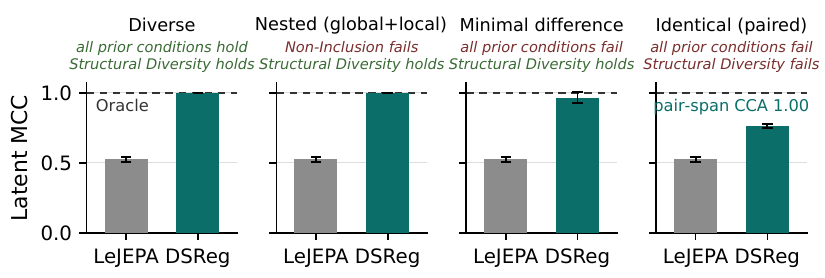}
        \caption{Footprint regimes ($N=16$; full sweep in Figure~\ref{fig:regime-scaling})}
        \label{fig:structdiv-regimes}
    \end{subfigure}
    \caption{\textbf{DSReg recovers individual latents wherever Structural Diversity holds, and only there.} (a)~DSReg (solid) is near ceiling at every $N$, twenty runs. LeJEPA (dashed) stays mixed. (b)~On analytic orbits $h=Qz$, recovery succeeds wherever the condition holds and falls back to the shared pair subspace where it fails.}
    \label{fig:synthetic-panel}
    \vspace{-4pt}
\end{figure}

The first question is whether the whole method works end to end. We train a generic encoder from scratch, fit the rotation from observations alone, and check whether individual latents come out. The synthetic benchmark makes Structural Diversity hold by construction, in the hardest form we consider. An MLP observation map $x_j=z_j+f_j(z_{<j})$ gives $z_k$ the footprint $\{x_k,\ldots,x_N\}$, so all footprints are distinct yet nested, the regime that prior conditions exclude. Figure~\ref{fig:synthetic-mcc} traces recovery across latent dimension. LeJEPA retains the predictive span but leaves its variables mixed, while DSReg lifts them to near-ceiling recovery at every dimension tested. Recovery tracks the linear-identifiability premise of Theorem~\ref{thm:lejepa-prelim}, which DSReg inherits rather than establishes. Once the encoder reaches the premise, the rotation does the rest. Latent-only rotations (PCA, Varimax \citep{kaiser1958varimax,rohe2020vintage}, FastICA) fail on the same estimated latents, so the Jacobian dependency signal rather than the optimizer is what does the work, on exactly the same inputs (Appendix~\ref{sec:app-baselines}, Table~\ref{tab:rotation-baselines}).

\vspace{-0.1em}
\subsection{Every regime lands where the theory puts it}
\label{sec:exp-regimes}
\vspace{-0.1em}

The theory makes a sharp and testable prediction. Recovery must succeed whenever footprints differ, however slightly, and may fail only when they coincide. Each panel of Figure~\ref{fig:structdiv-regimes} is an analytic orbit ($h=Qz$) realizing one regime of Proposition~\ref{prop:diversity-weaker}: \emph{diverse} footprints satisfy the prior conditions, \emph{nested} footprints add a global latent containing all local ones, \emph{minimal-difference} footprints differ in one observed feature, and \emph{identical} footprints violate Structural Diversity outright.

Wherever Structural Diversity is satisfied, recovery is near-perfect at every scale while unrotated baselines decay with dimension. The nested regime, which violates Non-Inclusion and both sparsity certificates, is recovered as cleanly as the easy one. Even the failure is the one the theory dictates. With identical footprints, individual recovery is impossible from support alone, yet DSReg still recovers the pair's span almost perfectly. The map is the theory's (Appendix Figures~\ref{fig:regime-scaling} and~\ref{fig:optimizer-ablation}).

\vspace{-0.1em}
\subsection{Sparse modules can act through recovered latents}
\label{sec:exp-sparse-use}
\vspace{-0.1em}

One may wonder why individual latents should be preferred to a well-spanned mixture. The modules built on a world model are sparse, and a rotation breaks their interface, so one estimated latent moves several physical factors at once. We therefore evaluate four sparse-use settings on states $h=Qz$ for a random orthogonal $Q$: visual edits in three environments, sparse model predictive control (MPC) fit from few transitions, short horizon rollout prediction, and transition-surprise detection, six probes in total, with environments and probe procedures detailed in Appendix~\ref{sec:app-sparse-use-details}.

DSReg improves every probe and wins in nearly every run (Figure~\ref{fig:sparse-use-summary}), while dense-span checks stay matched (Appendix Table~\ref{tab:dense-sparse-sanity}). The same signature persists on encoders trained from pixels, where DSReg is ahead in every single run of all three probes (Appendix~\ref{sec:app-learned-details}).

\vspace{-0.1em}
\subsection{The gains survive learned encoders and external renderers}
\label{sec:exp-visual-encoders}
\vspace{-0.1em}

\begin{wrapfigure}{r}{0.51\textwidth}
    \vspace{-12pt}
    \centering
    \includegraphics[width=\linewidth]{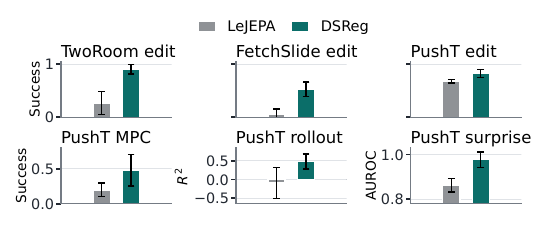}
    \vspace{-14pt}
    \captionsetup{skip=3pt} \caption{\textbf{DSReg improves all six sparse-use probes.} DSReg in teal, LeJEPA in gray. Tasks are (1) editing (top) and (2) control, prediction, and monitoring (bottom).}
    \label{fig:sparse-use-summary}
    \vspace{-5pt}
\end{wrapfigure}

Nothing in the method depends on the analytic setting, so we move to encoders trained from pixels and renderers we did not design (Appendix~\ref{sec:app-learned-details}). Figure~\ref{fig:learned-baselines} shows the outcome on the visual-factor and object-scene benchmarks. DSReg lifts convolutional encoders trained from scratch to the supervised Procrustes oracle that sees labels, while PCA, Varimax, and FastICA leave the latents nearly as mixed as no rotation, the rotational symmetry at work (Appendix~\ref{sec:app-baselines}). Figure~\ref{fig:learned-width} widens the representation. The top-$N$ right-singular subspace of the estimated Jacobians selects the directions the observations depend on, and the match holds at every width, so one can start wide and let the criterion find the subspace it needs.

\looseness=-1 On external renderers the same signature holds. Gaussian 3DShapes \citep{kim2018disentangling} rises from mixed to near-ceiling recovery at matched dense $R^2$ (Appendix Table~\ref{tab:shapes3d-gaussian}), and on quarter-orientation dSprites \citep{matthey2017dsprites}, where the premise is only partially attained, DSReg still beats swept $\beta$-VAE and $\beta$-TCVAE baselines and the premise sets the ceiling (Appendix Fig.~\ref{fig:dsprites-quarter}).

\begin{figure}[H]
    \centering
    \begin{subfigure}[b]{0.475\textwidth}
        \centering
        \includegraphics[width=\linewidth]{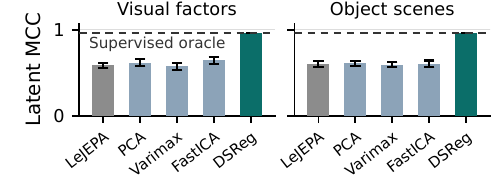}
        \caption{Rotation baselines on learned encoders}
        \label{fig:learned-baselines}
    \end{subfigure}\hspace{0.03\textwidth}
    \begin{subfigure}[b]{0.330\textwidth}
        \centering
        \includegraphics[width=\linewidth]{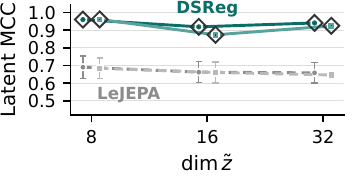}
        \caption{Overparameterized: $\dim\tilde z$ grows}
        \label{fig:learned-width}
    \end{subfigure}
    \caption{\textbf{DSReg matches the supervised Procrustes oracle on encoders trained from pixels.} (a)~Latent-only rotations leave the variables mixed. DSReg closes the gap to the label-fitted oracle (fifteen seeds). (b)~With $\dim z=8$ fixed, the match holds within $0.001$ as the estimate widens to $\dim\tilde z=32$ (five seeds).}
    \label{fig:learned-skyline}
    \vspace{-4pt}
\end{figure}

\FloatBarrier

\vspace{-0.3em}
\section{Discussion}
\vspace{-0.2em}

\label{sec:discussion}
Individual world latents can be recovered without reconstruction. To the best of our knowledge, prior guarantees at this resolution ran through a decoder or auxiliary access, and objectives free of any such anchor stopped at a linear class. This paper closes that gap: under the LeJEPA premise and Structural Diversity, \textbf{\emph{DSReg}} provably recovers each individual world latent through a sparsity regularization on the dependency Jacobian. The condition is strictly weaker than prior structural conditions, and the guarantee is stable to estimation error. The experiments confirm the theory where it is testable and show why it matters. After the DSReg rotation, a controller, a dynamics model, or a monitor acts through individual variables, while dense readouts remain unchanged on every benchmark.

A world model whose variables are the variables of the world is not a predictive black box. Its latents can be intervened on and inspected one at a time, and all of it comes from structure the observations already carry. Structural Diversity is a mild requirement: factors with identical footprints cannot be distinguished without extra information, and the real world might not produce them. Other conditions may also be leveraged to conduct alternative strategies to recover individual world latents, such as the orthogonality of the Jacobian in independent mechanism analysis \citep{gresele2021independent}.

The clearest limitation is scope: our benchmarks are simulated or rendered, and DSReg has not yet been tested on a physical robot. That gap is the future work that excites us most. How does a robot behave when each factor of its world model is its own variable, and how far does that carry planning? Identical footprints, the one boundary that support alone cannot cross, invite signals such as temporal structure or cheap interventions. Moreover, the Gaussian latent-world premise is a foothold rather than a ceiling. Every extension of linear identifiability to richer worlds widens the guarantee.

\section*{AI Disclosure}

In this paper, we used generative AI tools to assist polishing and plotting. We have reviewed all AI-assisted work. We take responsibility for the final content of this work, including text, claims or artifacts produced with the aid of generative AI.

\begingroup
\small
\setlength{\bibsep}{1pt plus 0.3pt}
\bibliography{references}
\bibliographystyle{iclr2026_conference}
\endgroup

\newpage
\appendix
\thmspacingfloor
\startcontents[appendix]
\printcontents[appendix]{}{1}{%
  \section*{Contents of the appendix}%
  \setcounter{tocdepth}{2}%
} \bigskip

\section{Proofs}
\label{app:proofs}

\paragraph{The linear-identifiability premise.} The proofs take the linear-identifiability theorem of \citet{klindt2026worldmodel} as their starting point. Its assumptions, in our notation, are: (i)~the world latents are independent Gaussians, $z\sim\mathcal N(0,I_d)$; (ii)~positive pairs are generated by a stationary additive-noise transition that preserves this marginal, $z'=\rho z+\sqrt{1-\rho^2}\,\eta$ with $\eta\sim\mathcal N(0,I_d)$ independent of $z$ and $\rho\in(0,1)$; (iii)~observations $x=g(z)$ come from an unknown generating process, and the estimated latents $h(z)=f_\theta(g(z))$ are given by a measurable map $h:\mathbb R^d\to\mathbb R^d$; (iv)~the learned representation satisfies the Gaussian constraint $h(z)\sim\mathcal N(0,I_d)$, as enforced by SIGReg \citep{balestriero2025lejepa}, and attains the alignment lower bound $2(1-\rho)d$ of Theorem~\ref{thm:lejepa-prelim}. Under (i)--(iv), the equality case of Theorem~\ref{thm:lejepa-prelim} gives $h(z)=Qz$ for some $Q\in O(d)$. Every statement below operates inside this Gaussian latent-world premise: it fixes the orthogonal equivalence class, and DSReg selects the signed permutation within it, from the observed dependency structure alone.

\paragraph{Proof sketch for linear identifiability in \cite{klindt2026worldmodel}.} So that the paper is self-contained on its premise, we sketch why the alignment bound holds with equality exactly at linear maps; the full proof is in \citet{klindt2026worldmodel}. Since $h(z)\sim\mathcal N(0,I_d)$ and the pair $(z,z^+)$ is stationary,
\[
\mathcal L_{\mathrm{align}}(h) = \mathbb E\|h(z^+)\|_2^2 + \mathbb E\|h(z)\|_2^2 - 2\sum_{i=1}^d \mathbb E\big[h_i(z^+)\,h_i(z)\big] = 2d - 2\sum_{i=1}^d \mathbb E\big[h_i(z^+)\,h_i(z)\big],
\]
so minimizing alignment maximizes the summed cross-correlations. Expand each component in the multivariate Hermite basis $\{H_\alpha\}$, orthonormal for the standard Gaussian: $h_i=\sum_{|\alpha|\ge1}c_{i\alpha}H_\alpha$ with $\sum_\alpha c_{i\alpha}^2=1$, where the Gaussian constraint supplies the zero mean and unit variance. The additive-noise transition acts diagonally on this basis (the Mehler kernel): $\mathbb E[H_\alpha(z^+)H_\beta(z)]=\rho^{|\alpha|}\,\mathbf 1\{\alpha=\beta\}$. Therefore
\[
\mathbb E\big[h_i(z^+)\,h_i(z)\big] = \sum_{|\alpha|\ge1}\rho^{|\alpha|}\,c_{i\alpha}^2 \;\le\; \rho ,
\]
with equality exactly when all coefficient mass sits at degree one, since $\rho^{k}<\rho$ for every $k\ge2$ and $\rho\in(0,1)$. Summing over $i$ gives $\mathcal L_{\mathrm{align}}(h)\ge 2(1-\rho)d$, and equality forces each $h_i$ to be a degree-one polynomial, so $h(z)=Az$ almost surely for some matrix $A$; the constraint $h(z)\sim\mathcal N(0,I_d)$ then gives $AA^\top=I_d$, hence $A\in O(d)$. To summarize, among unit-variance functions of a Gaussian, linear functions are the most predictable across the stationary transition, and the Gaussian constraint turns the optimal linear map into a rotation, exactly the premise our proofs consume.

We now give proof details for the identifiability statements of the main text. The JEPA-specific step is a bridge from linear identifiability to a concrete matrix problem: once a JEPA representation is identified up to an orthogonal transformation, DSReg must resolve the remaining orthogonal ambiguity in the dependency Jacobian. The structural steps below are stated directly in that matrix language, so the proofs are self-contained and do not revisit the JEPA objective itself.

\paragraph{Proof notation.} For an orthogonal matrix $U\in O(d)$, write
\[
    \mathcal I_j(U)=\{i:U_{ij}\ne0\}
\]
for the row support of column $j$. When the matrix is clear, we write $\mathcal I_j$. A column is a \emph{singleton} if $|\mathcal I_j|=1$ and \emph{mixed} if $|\mathcal I_j|\ge2$. In the main minimality proof, the set of mixed columns is
\[
    \mathcal J=\{j:|\mathcal I_j|\ge2\}, \qquad \mathcal I=\bigcup_{j\in\mathcal J}\mathcal I_j ,
\]
so $\mathcal I$ is the set of true-latent rows that participate in at least one mixed candidate latent.

\subsection{Dependency supports under the LeJEPA result}

\dependencybridge*

\begin{proof}
We prove the identity by writing the change of variables explicitly.

\emph{Step 1: replace the almost-sure identity by an a.e. representative.} The LeJEPA premise gives
\[
    h(z)=Qz \qquad \mu\text{-a.s.}
\]
for some $Q\in O(d)$. Changing $h$ on a $\mu$-null set does not change any a.e. dependency support, because the support criterion only asks whether an entry is nonzero on a set of positive $\mu$ measure. We may therefore take $h(z)=Qz$ on the differentiability points used below.

\emph{Step 2: compute the candidate latent map.} For a candidate DSReg rotation $R\in O(d)$,
\[
    \tilde z=Rh(z)=RQz .
\]
Since $R$ and $Q$ are orthogonal,
\[
    (RQ)^{-1}=Q^\top R^\top, \qquad z=Q^\top R^\top\tilde z .
\]

\emph{Step 3: apply the chain rule.} Let $D(z)=\partial x/\partial z$. At every point where $g$ is differentiable and the linear change of variables is defined,
\[
    \frac{\partial x}{\partial \tilde z} = \frac{\partial x}{\partial z} \frac{\partial z}{\partial \tilde z}.
\]
The first factor is $D(z)$, and the second factor is the constant matrix $Q^\top R^\top$. Thus
\[
    \frac{\partial x}{\partial \tilde z} = D(z)Q^\top R^\top .
\]
The orthogonal change of variables preserves Gaussian null sets, so this identity holds at the corresponding almost-everywhere differentiability points of the map $g$ as well.
\end{proof}

\begin{restatable}[Support rotations and dependency supports]{proposition}{supportrotations}
\label{prop:support-rotations}
Let $A$ be a fixed matrix, and apply a Givens rotation in the $(j,k)$ column plane with the angle chosen to zero an active entry in row $i$ of column $j$. Assume $A$ is \emph{generic for the triple $(i,j,k)$}: for every row $r\ne i$ with $A_{rj}A_{rk}\ne0$,
\[
    A_{rj}A_{ik}\ne A_{ij}A_{rk}
    \qquad\text{and}\qquad
    A_{rj}A_{ij}\ne -A_{rk}A_{ik}.
\]
These are finitely many polynomial equalities, so the matrices excluded for any triple form a finite union of measure-zero algebraic sets. Then the support of the rotated matrix changes only in the two rotated columns, and exactly one of the following four cases occurs:
\begin{enumerate}
\item \emph{Support Reduction}: the entry $(i,k)$ is active and the two column supports agree outside row $i$; the entry $(i,j)$ is deactivated and no other entry changes;
\item \emph{Reversible Acute Rotation}: the entry $(i,k)$ is active and the two column supports differ in exactly one row other than $i$; the entry $(i,j)$ is deactivated and exactly one previously inactive entry in the rotated columns becomes active;
\item \emph{Irreversible Acute Rotation}: the entry $(i,k)$ is active and the two column supports differ in at least two rows other than $i$; the entry $(i,j)$ is deactivated and all previously inactive entries in the rows where the supports differ become active;
\item \emph{Column Swap}: the entry $(i,k)$ is inactive; the supports of columns $j$ and $k$ are exchanged.
\end{enumerate}
Here a Givens rotation with $\sin\theta\cos\theta\ne0$ is called \emph{acute}, and an acute rotation applied to $A$ is called \emph{reversible} if some Givens rotation applied to the rotated matrix restores the support of $A$, and \emph{irreversible} otherwise; the terminology is from \citet{ghassami2020characterizing}. In particular, except for signed Column Swaps, a nontrivial Givens rotation changes the dependency support by deleting or adding active entries in the rotated columns.
\end{restatable}

\begin{proof}
Only the two rotated columns can change, so it is enough to analyze those columns row by row. Let $b_j$ and $b_k$ be columns $j$ and $k$ before the rotation. A Givens rotation in the $(j,k)$ column plane gives
\[
    b'_j=b_j\cos\theta+b_k\sin\theta, \qquad b'_k=-b_j\sin\theta+b_k\cos\theta .
\]
The angle is chosen so that the target entry $(i,j)$ vanishes:
\[
    b'_{ij} = b_{ij}\cos\theta+b_{ik}\sin\theta = 0, \qquad b_{ij}\ne0 .
\]

\emph{Case 1: the paired target entry is inactive.} If $b_{ik}=0$, then the equation above becomes $b_{ij}\cos\theta=0$. Since $b_{ij}\ne0$, we must have $\cos\theta=0$ and hence $\sin\theta=\pm1$. Therefore
\[
    b'_j=\pm b_k, \qquad b'_k=\mp b_j .
\]
The two column supports are exchanged. This is the Column Swap case.

\emph{Case 2: the paired target entry is active.} If $b_{ik}\ne0$, then
\[
    \tan\theta=-\frac{b_{ij}}{b_{ik}},
\]
so both $\sin\theta$ and $\cos\theta$ are nonzero. The target row after rotation is
\[
    (b'_{ij},b'_{ik})=(0,b'_{ik}).
\]
Because a Givens rotation preserves the Euclidean norm of the row pair,
\[
    |b'_{ik}|^2 = |b_{ij}|^2+|b_{ik}|^2 > 0.
\]
Thus the target row loses the entry in column $j$ and remains active in column $k$.

Now consider any other row $r\ne i$. The pair $(b_{rj},b_{rk})$ is multiplied by the same invertible two-dimensional rotation. If both entries are zero, they remain zero. If exactly one entry is nonzero, both entries of the image are nonzero, because $\sin\theta\cos\theta\ne0$. If both entries are nonzero, both entries of the image remain nonzero outside the excluded coefficient ratios. Indeed, an additional cancellation in column $j$ would require
\[
    b_{rj}\cos\theta+b_{rk}\sin\theta=0 \quad\Longleftrightarrow\quad \frac{b_{rj}}{b_{rk}}=\frac{b_{ij}}{b_{ik}},
\]
and an additional cancellation in column $k$ would require
\[
    -b_{rj}\sin\theta+b_{rk}\cos\theta=0 \quad\Longleftrightarrow\quad \frac{b_{rj}}{b_{rk}}=-\frac{b_{ik}}{b_{ij}} .
\]
Both conditions are exactly the polynomial equalities excluded by genericity for the triple $(i,j,k)$, so neither cancellation occurs.

Outside the exceptional set, the support outcome is determined only by how the two original column supports agree away from the target row:
\begin{enumerate}
\item If the supports agree on every row $r\ne i$, then no missing entry is created away from the target row. The only support change is deletion of $(i,j)$, giving Support Reduction.
\item If the supports differ on exactly one row $r\ne i$, then that row had exactly one active entry before the rotation and two active entries after it. The target deletion is balanced by one activation, giving a Reversible Acute Rotation.
\item If the supports differ on at least two rows $r\ne i$, then each such row gains the missing entry. The target deletion is accompanied by at least two activations, giving an Irreversible Acute Rotation.
\end{enumerate}
Together with the inactive-paired-entry case, these are the four cases in the proposition.
\end{proof}

\subsection{Identifiability under Structural Diversity}

We first record the auxiliary genericity facts and then prove Theorem~\ref{thm:idjepa-zheng}. The support-minimality proof works directly with the data-generating dependency Jacobian. Meanwhile, Functional no-cancellation (Assumption~\ref{asm:functional-no-cancellation}) rules out a.e. cancellations by a constant orthogonal mixing matrix, so a mixed candidate latent has the union of the source dependency footprints as its own.

\begin{proposition}[Functional no-cancellation within analytic families]
\label{prop:no-cancel-generic}
Let $\{F_\theta:\theta\in\Theta\}$ be a smooth observation family with $\Theta\subset\mathbb R^q$ open and connected. For each row and parameter let
\[
    I_r(\theta)=\{\,i:D_{ri,\theta}\ne0 \text{ on a set of positive }\mu\text{ measure}\,\}.
\]
Assume there are fixed sets $I_r$ such that $I_r(\theta)=I_r$ for all $\theta\in\Theta$. Suppose each active derivative $D_{ri,\theta}=\partial(F_\theta)_r/\partial z_i$, $i\in I_r$, is square integrable and the Gram determinant
\[
    G_r(\theta) = \det\!\left[ \langle D_{ri,\theta},D_{ri',\theta}\rangle_{L^2(\mu)} \right]_{i,i'\in I_r}
\]
is real analytic in $\theta$. If, for every row $r$, there exists $\theta_r\in\Theta$ with $G_r(\theta_r)\ne0$, then Assumption~\ref{asm:functional-no-cancellation} holds for Lebesgue-almost every $\theta\in\Theta$.
\end{proposition}

\begin{proof}[Proof of Proposition~\ref{prop:no-cancel-generic}]
Fix an observation row $r$. If $|I_r|\le1$, then any relation
\[
    \sum_{i\in I_r}c_iD_{ri,\theta}=0 \quad\text{in }L^2(\mu)
\]
forces all coefficients to be zero, because there is at most one active nonzero function.

Now suppose $|I_r|>1$. Let
\[
    H_r(\theta) = \left[ \langle D_{ri,\theta},D_{ri',\theta}\rangle_{L^2(\mu)} \right]_{i,i'\in I_r}
\]
be the row-wise Gram matrix. For any coefficient vector $c=(c_i)_{i\in I_r}$,
\[
    c^\top H_r(\theta)c = \left\| \sum_{i\in I_r}c_iD_{ri,\theta} \right\|_{L^2(\mu)}^2 .
\]
Therefore the active row functions are linearly dependent in $L^2(\mu)$ if and only if $H_r(\theta)$ is singular, equivalently
\[
    G_r(\theta)=\det H_r(\theta)=0 .
\]
By assumption, $G_r$ is real analytic and is not identically zero on the connected open set $\Theta$. A nonzero real analytic function has a Lebesgue-null zero set. Thus the no-cancellation condition fails for this row only on a measure-zero subset of $\Theta$. Taking the finite union over the observation rows gives the result. The argument is the same one used for faithfulness in \citet[Proposition~3]{ng2023sparseica}.
\end{proof}

\paragraph{An empirical diagnostic for no-cancellation.} The no-cancellation condition also admits a rotation-invariant empirical diagnostic. Specifically, for a fixed observation row $r$, the estimated derivative family is $D_{r,:}(\cdot)U$, an invertible recombination of the true row functions, so the functional row rank across anchors is invariant to the unknown rotation. The local Jacobians in Equation~(\ref{eq:idjepa-objective}) estimate this rank by a row-wise Gram matrix; a row whose estimated rank falls below its recovered row-support size signals a spurious cancellation and flags that row for closer inspection.

\begin{lemma}[Functional support of a mixed column]
\label{lem:functional-union-support}
For a vector-valued function $v(\cdot)$, write $\operatorname{supp}_{\mu}(v)=\{r:v_r(z)\ne0\text{ on a set of positive } \mu\text{ measure}\}$. Assume Functional no-cancellation (Assumption~\ref{asm:functional-no-cancellation}), let $U\in O(d)$, and define
\[
    \mathcal I_j=\{i:U_{ij}\ne0\}.
\]
Then the a.e. support of the $j$-th column of $D(\cdot)U$ is
\[
    \operatorname{supp}_{\mu}\big((D(\cdot)U)_{:,j}\big) = \bigcup_{i\in\mathcal I_j}\mathcal S_i .
\]
\end{lemma}

\begin{proof}
Fix a column $j$ and an observation row $r$. The row entry of the mixed column is
\[
    (D(z)U)_{rj} = \sum_{i=1}^d D_{ri}(z)U_{ij} = \sum_{i\in\mathcal I_j}D_{ri}(z)U_{ij}.
\]
We prove both inclusions.

\emph{First inclusion.} If
\[
    r\notin \bigcup_{i\in\mathcal I_j}\mathcal S_i ,
\]
then $r\notin\mathcal S_i$ for every $i\in\mathcal I_j$. By definition of $\mathcal S_i$, each $D_{ri}(\cdot)$ is zero $\mu$-a.e. Hence $(D(\cdot)U)_{rj}$ is zero $\mu$-a.e., so $r$ is not in the support of the mixed column, which proves the first inclusion.

\emph{Second inclusion.} Now suppose
\[
    r\in \bigcup_{i\in\mathcal I_j}\mathcal S_i .
\]
Let
\[
    K_{rj}=\{i\in\mathcal I_j:r\in\mathcal S_i\}.
\]
This set is nonempty. Terms with $i\notin K_{rj}$ are zero $\mu$-a.e., so
\[
    (D(z)U)_{rj} = \sum_{i\in K_{rj}}D_{ri}(z)U_{ij} \qquad \mu\text{-a.e.}
\]
For every $i\in K_{rj}$, the coefficient $U_{ij}$ is nonzero by the definition of $\mathcal I_j$. Thus the last display is a nontrivial constant-coefficient combination of the active row functions $\{D_{ri}(\cdot):i\in K_{rj}\}$. This family is a subfamily of the row-$r$ active family in Assumption~\ref{asm:functional-no-cancellation}. Functional no-cancellation says that no such nontrivial combination can be zero $\mu$-a.e. Therefore $(D(\cdot)U)_{rj}$ is nonzero on a set of positive $\mu$ measure, so $r$ belongs to the support of the mixed column, as required.
\end{proof}

\begin{lemma}[Functional support minimality]
\label{lem:functional-support-minimality}
Under the assumptions of Theorem~\ref{thm:idjepa-zheng}, every solution of
\[
    \min_{U\in O(d)}\|D(\cdot)U\|_{0,\mu}
\]
is a signed permutation matrix.
\end{lemma}

\begin{proof}
Let $\hat U$ be a minimizer. We show that $\hat U$ cannot contain any mixed column.

\emph{Step 1: if $\hat U$ is not a signed permutation, it has a nonempty mixed block.} For each column $j$, let
\[
    \mathcal I_j=\{i:\hat U_{ij}\ne0\}.
\]
Every column of an orthogonal matrix has unit norm, so $|\mathcal I_j|\ge1$. If $|\mathcal I_j|=1$, say $\mathcal I_j=\{r\}$, then
\[
    1=\|\hat U_{:,j}\|_2^2=\hat U_{rj}^2,
\]
so $\hat U_{rj}=\pm1$. Since row $r$ also has unit norm,
\[
    1=\|\hat U_{r,:}\|_2^2 =\hat U_{rj}^2+\sum_{\ell\ne j}\hat U_{r\ell}^2 =1+\sum_{\ell\ne j}\hat U_{r\ell}^2 .
\]
Hence $\hat U_{r\ell}=0$ for every $\ell\ne j$. A singleton column therefore uses its row completely.

This observation has two consequences. First, two singleton columns cannot use the same row. Second, a row used by a singleton column cannot appear in any mixed column. Indeed, if a singleton column uses row $r$, then the calculation above gives $\hat U_{r\ell}=0$ for every other column $\ell$. Therefore $r\notin\mathcal I_\ell$ for every mixed column $\ell$, which is precisely the second consequence.

If every column were a singleton, the $d$ singleton columns would occupy $d$ distinct rows, each with a single $\pm1$ entry. Then $\hat U$ would be a signed permutation matrix. Arguing by contradiction, suppose $\hat U$ is not a signed permutation. Then the set of mixed columns
\[
    \mathcal J=\{j:|\mathcal I_j|\ge2\}
\]
is nonempty. Let
\[
    \mathcal I=\bigcup_{j\in\mathcal J}\mathcal I_j .
\]
The row-completion argument above also shows
\[
    \hat U_{i\ell}=0 \quad \text{for every } i\in\mathcal I \text{ and every singleton column }\ell\notin\mathcal J. \tag{i}
\]
In words, rows touched by mixed columns have support only inside mixed columns.

\emph{Step 2: the mixed rows and mixed columns form a square orthogonal block.} Consider the submatrix
\[
    B=\hat U_{\mathcal I,\mathcal J}.
\]
By definition of $\mathcal I$, every nonzero entry of every mixed column $j\in\mathcal J$ lies in a row from $\mathcal I$. Restricting such a column to rows $\mathcal I$ therefore removes only zeros. Hence for $j,j'\in\mathcal J$,
\[
    \langle B_{:,j},B_{:,j'}\rangle = \langle \hat U_{:,j},\hat U_{:,j'}\rangle = \mathbf 1\{j=j'\}.
\]
The columns of $B$ are orthonormal. Since $B$ has $|\mathcal J|$ orthonormal columns in $\mathbb R^{|\mathcal I|}$, we must have
\[
    |\mathcal J|\le |\mathcal I|. \tag{ii}
\]

Now use (i). If $i\in\mathcal I$, then row $i$ has zero entries in every singleton column, and every column is either singleton or mixed. Thus restricting row $i$ to columns $\mathcal J$ removes only zeros. For $i,i'\in\mathcal I$,
\[
    \langle B_{i,:},B_{i',:}\rangle = \langle \hat U_{i,:},\hat U_{i',:}\rangle = \mathbf 1\{i=i'\}.
\]
The rows of $B$ are orthonormal. Since $B$ has $|\mathcal I|$ orthonormal rows in $\mathbb R^{|\mathcal J|}$, we must have
\[
    |\mathcal I|\le |\mathcal J|. \tag{iii}
\]
Combining (ii) and (iii) gives $|\mathcal I|=|\mathcal J|$. Therefore $B$ is square. Since its columns are orthonormal,
\[
    B^\top B=I_{|\mathcal J|},
\]
so $B$ is invertible. This is the precise meaning of saying that the mixed block is square orthogonal.

\emph{Step 3: invertibility gives a distinct representative for each mixed column.} Expand the determinant of $B$:
\[
    \det(B) = \sum_{\tau:\mathcal J\to\mathcal I\ \text{bijection}} \operatorname{sgn}(\tau)\prod_{j\in\mathcal J} B_{\tau(j),j}.
\]
Since $B$ is invertible, $\det(B)\ne0$. Thus at least one product in this sum is nonzero. For that bijection, call it $\sigma$, every factor $B_{\sigma(j),j}$ is nonzero. Equivalently,
\[
    \sigma(j)\in\mathcal I_j \qquad\text{for every }j\in\mathcal J.
\]
So the bijection $\sigma:\mathcal J\to\mathcal I$ selects, for each mixed column, one true latent row that the column actually uses, and it selects distinct rows for distinct mixed columns of the block.

\emph{Step 4: construct a signed-permutation competitor.} Define $\tilde U$ as follows. For each mixed column $j\in\mathcal J$, put a single nonzero entry $\pm1$ at row $\sigma(j)$. For each singleton column, keep the same singleton row and sign as in $\hat U$. The rows $\sigma(\mathcal J)=\mathcal I$ are disjoint from the singleton rows by Step~1, and the singleton rows are distinct. Therefore every row and every column of $\tilde U$ has exactly one nonzero entry of magnitude one. Hence $\tilde U$ is a signed permutation matrix.

\emph{Step 5: the signed-permutation competitor is never less sparse.} For a mixed column $j\in\mathcal J$, Lemma~\ref{lem:functional-union-support} gives
\[
    \operatorname{supp}_{\mu}\big((D(\cdot)\hat U)_{:,j}\big) = \bigcup_{i\in\mathcal I_j}\mathcal S_i .
\]
The competitor $\tilde U$ makes column $j$ pure row $\sigma(j)$, so
\[
    \operatorname{supp}_{\mu}\big((D(\cdot)\tilde U)_{:,j}\big) = \mathcal S_{\sigma(j)} .
\]
Because $\sigma(j)\in\mathcal I_j$, one set in the union is $\mathcal S_{\sigma(j)}$. Therefore
\[
    \left| \bigcup_{i\in\mathcal I_j}\mathcal S_i \right| \ge |\mathcal S_{\sigma(j)}| = \|(D(\cdot)\tilde U)_{:,j}\|_{0,\mu}. \tag{iv}
\]
Singleton columns are unchanged, so their support counts are equal under $\hat U$ and $\tilde U$. Summing (iv) over mixed columns and adding the singleton columns gives
\[
    \|D(\cdot)\hat U\|_{0,\mu} \ge \|D(\cdot)\tilde U\|_{0,\mu}. \tag{v}
\]

\emph{Step 6: Structural Diversity makes the inequality strict.} Suppose equality held in (v). Then equality must hold in (iv) for every mixed column:
\[
    \bigcup_{i\in\mathcal I_j}\mathcal S_i = \mathcal S_{\sigma(j)} \qquad \text{for every }j\in\mathcal J . \tag{vi}
\]
Since every set in a union is contained in the union, (vi) implies
\[
    \mathcal S_k\subseteq \mathcal S_{\sigma(j)} \qquad \text{for every }k\in\mathcal I_j . \tag{vii}
\]

Choose an inclusion-minimal footprint among $\{\mathcal S_i:i\in\mathcal I\}$, and call it $\mathcal S_{i_*}$. Inclusion-minimal means that there is no $i\in\mathcal I$ such that
\[
    \mathcal S_i\subsetneq \mathcal S_{i_*}. \tag{viii}
\]
The bijection $\sigma:\mathcal J\to\mathcal I$ is onto, so there is a mixed column $j_*\in\mathcal J$ with
\[
    \sigma(j_*)=i_* .
\]
Because $j_*$ is mixed, $\mathcal I_{j_*}$ has at least two elements. Also $i_*=\sigma(j_*)\in\mathcal I_{j_*}$. Therefore we can choose
\[
    k\in\mathcal I_{j_*}, \qquad k\ne i_* .
\]
Applying (vii) to column $j_*$ gives
\[
    \mathcal S_k\subseteq\mathcal S_{\sigma(j_*)} = \mathcal S_{i_*}. \tag{ix}
\]
By (ix) we have $\mathcal S_k\subseteq\mathcal S_{i_*}$, so exactly one of two cases holds. If $\mathcal S_k=\mathcal S_{i_*}$ (in particular if $\mathcal S_{i_*}=\emptyset$, which forces $\mathcal S_k=\emptyset$), then two different latents have the same footprint, contradicting Structural Diversity. If $\mathcal S_k\subsetneq\mathcal S_{i_*}$, then (viii) is violated, contradicting the inclusion-minimal choice of $\mathcal S_{i_*}$. Hence equality in~(v) is impossible in either one of the two cases above.

Thus
\[
    \|D(\cdot)\hat U\|_{0,\mu} > \|D(\cdot)\tilde U\|_{0,\mu},
\]
which contradicts the optimality of $\hat U$. Therefore a minimizer cannot have a mixed column. Every column is singleton, and orthogonality then forces the matrix to be a signed permutation.
\end{proof}

\begin{remark}[Global minimum of the support criterion]
\label{rem:global-minimum}
Steps~1--5 above do not use Structural Diversity: for every $U\in O(d)$ they construct a signed-permutation competitor and yield
\[
    \|D(\cdot)U\|_{0,\mu}\ \ge\ \sum_{i=1}^d|\mathcal S_i|,
\]
with equality attained by every signed permutation. Hence $\sum_{i=1}^d|\mathcal S_i|$ is the global minimum of the support criterion and signed permutations are always among its minimizers; Structural Diversity, invoked only in Step~6, is exactly what removes the non-permutation minimizers. This is the formal basis for the sharp-boundary discussion after Proposition~\ref{prop:diversity-weaker}: when $\mathcal S_i=\mathcal S_j$ for some $i\ne j$, a Givens rotation in that plane produces a non-permutation matrix that still attains $\sum_i|\mathcal S_i|$.
\end{remark}

\idjepazheng*

\begin{proof}
The proof reduces the DSReg criterion over $R$ to the matrix problem solved by Lemma~\ref{lem:functional-support-minimality}.

\emph{Step 1: express the candidate Jacobian through one orthogonal matrix.} By the LeJEPA premise, $h(z)=Qz$ $\mu$-a.e. for some $Q\in O(d)$. As in Lemma~\ref{lem:dependency-bridge}, changing $h$ on a null set does not change the a.e. support, so we use the representative $h(z)=Qz$. For candidate latents $\tilde z=Rh(z)$, the dependency bridge gives
\[
    \frac{\partial x}{\partial \tilde z} = D(z)Q^\top R^\top .
\]
Since $Q,R\in O(d)$, the following matrix is again orthogonal:
\[
    U=Q^\top R^\top .
\]

\emph{Step 2: show that optimizing over $R$ is the same as optimizing over $U$.} The map
\[
    R\mapsto U=Q^\top R^\top
\]
is a bijection from $O(d)$ to $O(d)$. Indeed, for any $U\in O(d)$, choosing
\[
    R=(QU)^\top
\]
gives
\[
    Q^\top R^\top = Q^\top (QU) = U .
\]
Therefore
\[
    \min_{R\in O(d)} \left\| \frac{\partial x}{\partial \tilde z} \right\|_{0,\mu} = \min_{U\in O(d)} \|D(\cdot)U\|_{0,\mu}.
\]
A rotation $R$ minimizes the left-hand side exactly when its induced $U$ minimizes the right-hand side.

\emph{Step 3: apply support minimality.} Lemma~\ref{lem:functional-support-minimality} says every minimizer of the right-hand side is a signed permutation. Write such a minimizer as
\[
    U=SP,
\]
where $S$ is diagonal with entries in $\{\pm1\}$ and $P$ is a permutation matrix. Since $U=Q^\top R^\top$, we have
\[
    RQ=U^\top=P^\top S .
\]
Thus $RQ$ is a signed permutation.

\emph{Step 4: translate the matrix statement into latent recovery.} The candidate latents are
\[
    \tilde z = Rh(z) = RQz = P^\top S z .
\]
Hence each component of $\tilde z$ is one component of $z$, up to a sign and relabeling:
\[
    \tilde z_i=s_i z_{\pi(i)} \qquad\mu\text{-a.e.}
\]
This is the claimed signed-permutation identifiability of Definition~\ref{def:identifiability}.
\end{proof}

\subsection{Comparison of structural conditions}

\diversityweaker*

\begin{proof}
The proposition compares conditions only through column-support sets, so we can reason directly at the level of the footprints $\mathcal S_1,\ldots,\mathcal S_d$ throughout.

\emph{Step 1: prior structural conditions imply Structural Diversity.} \citet[Theorem~2]{ng2023sparseica} establish that the overlap-rank form of Structural Sparsity implies Non-Inclusion (termed Column Subset there), and Non-Inclusion implies Structural Variability; \citet[Theorem~3]{ng2023sparseica} establish the same implication chain starting from the intersection form of Structural Sparsity. Hence
\[
    \text{Structural Sparsity (either form)} \;\Longrightarrow\; \text{Non-Inclusion} \;\Longrightarrow\; \text{Structural Variability}.
\]
It remains to check that Structural Variability implies Structural Diversity.

Structural Variability says
\[
    |\mathcal S_i\triangle\mathcal S_j|\ge2 \qquad\text{for all }i\ne j .
\]
If two footprints were equal, $\mathcal S_i=\mathcal S_j$, then their symmetric difference would be empty:
\[
    \mathcal S_i\triangle\mathcal S_j=\emptyset, \qquad |\mathcal S_i\triangle\mathcal S_j|=0,
\]
contradicting Structural Variability. Hence all footprints are distinct: Structural Diversity holds.

\emph{Step 2: Structural Diversity does not imply the prior conditions.} It suffices to give one support pattern satisfying Structural Diversity but violating the stronger conditions. Let $d=2$, $p=2$, and
\[
    \mathcal S_1=\{1\}, \qquad \mathcal S_2=\{1,2\}.
\]
The footprints are distinct, so Structural Diversity holds. They are realized, for example, by
\[
    F_1(z)=z_1+\sin z_2, \qquad F_2(z)=z_2+\sin z_2 .
\]
Indeed,
\[
    \frac{\partial F_1}{\partial z_1}=1,\quad \frac{\partial F_1}{\partial z_2}=\cos z_2,\quad \frac{\partial F_2}{\partial z_1}=0,\quad \frac{\partial F_2}{\partial z_2}=1+\cos z_2,
\]
so the first, second, and fourth derivatives are nonzero on sets of positive measure while the third vanishes identically, giving exactly the two footprints above. Moreover, the active partial derivatives in observation row $1$, namely $\partial F_1/\partial z_1=1$ and $\partial F_1/\partial z_2=\cos z_2$, are linearly independent in $L^2(\mu)$, so this witness also satisfies Functional no-cancellation (Assumption~\ref{asm:functional-no-cancellation}) and Theorem~\ref{thm:idjepa-zheng} therefore applies to it without modification.

However,
\[
    \mathcal S_1\subsetneq\mathcal S_2,
\]
so Non-Inclusion fails. Also
\[
    \mathcal S_1\triangle\mathcal S_2=\{2\}, \qquad |\mathcal S_1\triangle\mathcal S_2|=1,
\]
so Structural Variability fails. Since both structural-sparsity forms imply Non-Inclusion through the cited implication chains, they fail as well. Hence Structural Diversity is strictly weaker.
\end{proof}

\subsection{Robustness to Jacobian perturbation}
\label{app:robustness}

This section proves the stability results of Section~\ref{sec:practical-estimator}. The first lemma removes general invertible-linear ambiguity at the population level; the later lemmas bound the constants the theorem uses.

\begin{lemma}[Whitening reduction]
\label{lem:whitening-reduction}
Let $z\sim\mathcal N(0,I_d)$ and $h(z)=Az$ for an invertible $A\in\mathbb R^{d\times d}$. Then $W=\operatorname{Cov}(h)^{-1/2}=(AA^\top)^{-1/2}$ satisfies $Wh(z)=\tilde Qz$ with $\tilde Q=(AA^\top)^{-1/2}A\in O(d)$.
\end{lemma}

\begin{proof}
The covariance of the mixed representation is
\[
    \operatorname{Cov}(h)=A\operatorname{Cov}(z)A^\top=AA^\top .
\]
This matrix is symmetric positive definite, so its inverse square root exists and is symmetric. For $\tilde Q=(AA^\top)^{-1/2}A$,
\[
    \tilde Q\tilde Q^\top=(AA^\top)^{-1/2}AA^\top(AA^\top)^{-1/2}=I_d ,
\]
and likewise
\[
    \tilde Q^\top\tilde Q=A^\top(AA^\top)^{-1}A=I_d ,
\]
so $\tilde Q$ is orthogonal. Finally,
\[
    Wh(z)=(AA^\top)^{-1/2}Az=\tilde Qz .
\]
\end{proof}

\paragraph{Constants of the support pattern.} Since $\mu$ is a probability measure, the bounded-rows condition of Theorem~\ref{thm:approx-recovery} places every $D_{ri}$ in $L^2(\mu)$, and under Assumption~\ref{asm:functional-no-cancellation} the active row functions $\{D_{ri}\}_{i\in I_r}$ are linearly independent as elements of $L^2(\mu)$, so each row Gram matrix $G_r=[\langle D_{ri},D_{ri'}\rangle_{L^2(\mu)}]_{i,i'\in I_r}$ is positive definite. Write $\sigma_r=\lambda_{\min}(G_r)^{1/2}>0$ and $\sigma=\min\{\sigma_r:I_r\ne\emptyset\}$; under Structural Diversity with $d\ge2$ at most one footprint is empty, so some $I_r$ is nonempty and $\sigma$ is well defined. For $U\in O(d)$ let $v_{rj}(U)=\|U_{I_r,j}\|_2$ (zero when $I_r=\emptyset$); each $v_{rj}$ is continuous on $O(d)$, and by Lemma~\ref{lem:functional-union-support} the population count satisfies $N(U)=\|D(z)U\|_{0,\mu}=|\{(r,j):v_{rj}(U)>0\}|$. Let $SP(d)$ denote the signed permutation matrices, and for $\delta>0$ let $K_\delta=\{U\in O(d):\min_{P\in SP(d)}\|U-P\|_F\ge\delta\}$. Define $\rho^*(U)$ as the $(s^*{+}1)$-th largest of the $pd$ values $v_{rj}(U)$, where $s^*=\sum_i|\mathcal S_i|$, and $\rho^*(\delta)=\min_{U\in K_\delta}\rho^*(U)$, with $\rho^*(\delta)=+\infty$ when $K_\delta=\emptyset$ (making Theorem~\ref{thm:approx-recovery} trivial).

\begin{lemma}[Quantitative activity]
\label{lem:quantitative-activity}
Let Assumption~\ref{asm:functional-no-cancellation} and the bounded-rows condition of Theorem~\ref{thm:approx-recovery} hold, let $r$ be a row with $I_r\ne\emptyset$, let $c\in\mathbb R^{I_r}$ with $\|c\|_2\le1$, and set $f_c=\sum_{i\in I_r}c_iD_{ri}$. Then for every $0\le t<\sigma_r\|c\|_2$,
\[
\mu\big(|f_c|>t\big)\;\ge\;\frac{\sigma_r^2\|c\|_2^2-t^2}{M^2}\;>\;0 .
\]
\end{lemma}

\begin{proof}
\emph{Step 1: pointwise upper bound.} By Cauchy--Schwarz, for $\mu$-a.e.\ $z$,
\[
    |f_c(z)|\le\|D_{r,:}(z)\|_2\|c\|_2\le M .
\]

\emph{Step 2: second-moment lower bound.} The Gram bound gives
\[
    \|f_c\|_{L^2(\mu)}^2=c^\top G_rc\ge\sigma_r^2\|c\|_2^2 .
\]

\emph{Step 3: split the integral.} Combining the two bounds,
\[
\begin{aligned}
    \sigma_r^2\|c\|_2^2 &\le \int|f_c|^2\,d\mu \\
    &\le t^2\,\mu(|f_c|\le t)+M^2\,\mu(|f_c|>t) \\
    &\le t^2+M^2\,\mu(|f_c|>t).
\end{aligned}
\]
Rearranging the resulting inequality proves the claimed lower bound.
\end{proof}

\begin{lemma}[Uniform pattern margin]
\label{lem:pattern-margin}
Let Assumptions~\ref{asm:structural-diversity} and~\ref{asm:functional-no-cancellation} hold with $d\ge2$, and let $\delta>0$ be such that $K_\delta\ne\emptyset$. Then $\rho^*(\delta)>0$.
\end{lemma}

\begin{proof}
\emph{Step 1: reduce to pointwise positivity.} Each $v_{rj}$ is continuous on $O(d)$, so the order statistic $\rho^*(U)$ (the $(s^*{+}1)$-th largest of finitely many continuous functions) is continuous, and $K_\delta$ is a closed subset of the compact group $O(d)$, hence compact. A function that is positive and continuous on a compact set attains a strictly positive minimum over that whole set, so it suffices to show $\rho^*(U)>0$ for every $U\in K_\delta$.

\emph{Step 2: the population count exceeds the minimum.} Fix $U\in K_\delta$; it is not a signed permutation. By Remark~\ref{rem:global-minimum}, $\min_{U'}N(U')=s^*$ and every signed permutation attains it, and by Lemma~\ref{lem:functional-support-minimality} any $U$ with $N(U)=s^*$ is a signed permutation; both statements hold under Assumptions~\ref{asm:structural-diversity} and~\ref{asm:functional-no-cancellation} alone. Since the count is integer-valued,
\[
    N(U)\ge s^*+1 .
\]

\emph{Step 3: convert the count into a margin.} By the union-support identity,
\[
    N(U)=|\{(r,j):v_{rj}(U)>0\}| ,
\]
at least $s^*+1$ of the values $v_{rj}(U)$ are strictly positive, i.e.\ $\rho^*(U)>0$. Under Structural Diversity with $d\ge2$ not every footprint equals $[p]$, so $s^*+1\le pd$ and $\rho^*(U)$ is well defined. By Step~1, $\rho^*(\delta)>0$ as claimed.
\end{proof}

\begin{remark}[The constants on the witness]
\label{rem:constants-witness}
It is worth noting that the constants are not vacuous. For the two-dimensional witness of Proposition~\ref{prop:diversity-weaker} ($F_1=z_1+\sin z_2$, $F_2=z_2+\sin z_2$), a direct computation gives $\sigma\approx0.374$ and $\rho^*(0.4)\approx0.28$, so the recovery condition $\tau+\varepsilon<\sigma\rho^*(\delta)$ of Theorem~\ref{thm:approx-recovery} is satisfiable at nontrivial noise levels on an example that meets every assumption.
\end{remark}

\robusttheorem*

\begin{proof}
\emph{Step 1: change of variables.} Throughout write $U=(RQ)^\top$. By Lemma~\ref{lem:dependency-bridge},
\[
    \frac{\partial x}{\partial\tilde z}=D(z)U .
\]
The map $R\mapsto U$ is a bijection of $O(d)$, and
\[
    \min_P\|RQ-P\|_F=\min_P\|U-P\|_F
\]
because transposition is an isometry and $SP(d)$ is closed under transposition. Moreover, $\widehat N_\tau$ is integer-valued with finitely many attainable values, so its minimum over $O(d)$ is attained.

\emph{Step 2: uniform bound on the perturbation.} For any fixed $U$ and $\mu$-a.e.\ $z$, Cauchy--Schwarz with unit columns gives
\[
    |(E(z)U)_{rj}|\le\|E_{r,:}(z)\|_2\le\varepsilon \qquad\text{for all }(r,j).
\]
The exceptional null set may depend on $U$, and each step below fixes one $U$, so the $U$-dependence of the null set costs nothing.

\emph{Step 3: value at signed permutations.} Choose $R_0=P_0Q^\top$ for any $P_0\in SP(d)$, so $U_0=(R_0Q)^\top\in SP(d)$. If entry $(r,j)$ is counted by $\widehat N_\tau(R_0)$, then
\[
    |(\widehat D(z)U_0)_{rj}|>\tau
\]
on a set of positive measure. Since $|(E(z)U_0)_{rj}|\le\varepsilon$ a.e., on that set also
\[
    |(D(z)U_0)_{rj}|>\tau-\varepsilon>0 ,
\]
hence $(r,j)$ is active for $D(z)U_0$. Therefore
\[
    \widehat N_\tau(R_0)\le N(U_0)=s^* ,
\]
using Remark~\ref{rem:global-minimum} for the last equality.

\emph{Step 4: value far from signed permutations.} Let $R$ satisfy $\min_P\|RQ-P\|_F\ge\delta$, so $U=(RQ)^\top\in K_\delta$. By the definition of $\rho^*(\delta)$, and Lemma~\ref{lem:pattern-margin} for its positivity, at least $s^*+1$ entries $(r,j)$ have
\[
    v_{rj}(U)\ge\rho^*(\delta) .
\]
Fix such an entry and let $c=U_{I_r,j}$. The inactive terms $\sum_{i\notin I_r}U_{ij}D_{ri}(z)$ vanish $\mu$-a.e., so for $\mu$-a.e.\ $z$,
\[
    (D(z)U)_{rj}=f_c(z) ,
\]
with $\|c\|_2=v_{rj}(U)\ge\rho^*(\delta)$ and $\|c\|_2\le1$ as a sub-vector of a unit column. Applying Lemma~\ref{lem:quantitative-activity} with $t=\tau+\varepsilon<\sigma\rho^*(\delta)\le\sigma_r\|c\|_2$ gives
\[
    \mu(|(D(z)U)_{rj}|>\tau+\varepsilon)>0 .
\]
On that set,
\[
    |(\widehat D(z)U)_{rj}|\ge|(D(z)U)_{rj}|-\varepsilon>\tau ,
\]
so $(r,j)$ is counted by $\widehat N_\tau(R)$. Hence $\widehat N_\tau(R)\ge s^*+1$ holds for every rotation $R$ in this set.

\emph{Step 5: combine.} For every such $R$,
\[
    \widehat N_\tau(R)\ge s^*+1>s^*\ge\widehat N_\tau(R_0) ,
\]
so no such $R$ can minimize $\widehat N_\tau$.
\end{proof}

\FloatBarrier
\section{Additional discussions}
\label{app:related}
Two literatures meet for our goal. Methods that recover individual latents anchor them to observations through a decoder or a likelihood, through auxiliary variables or interventions, or through an objective built around the true positive-pair conditional. In contrast, methods that dispense with all of these identify the latent state at most up to an orthogonal transformation, leaving the individual latents mixed. The three groups below trace each side and the premise that connects them.

\paragraph{Nonlinear ICA and self-supervised identifiability.} Identifiability has been studied most extensively in nonlinear ICA, where latent sources cannot be recovered from i.i.d.\ observations alone. One family of results restores identifiability through auxiliary structure: nonstationary segments \citep{hyvarinen2016unsupervised}, temporal dependence \citep{hyvarinen2017pcl,klindt2021slowvae}, conditioning on a general auxiliary variable \citep{hyvarinen2019nonlinear}, and conditional priors in deep generative models \citep{khemakhem2020ivae}. A second family analyzes self-supervised objectives directly. Contrastive learning provably inverts the data-generating process when the positive-pair conditional matches the contrastive similarity \citep{zimmermann2021contrastive}, and augmentation-based learning block-identifies the invariant content variables \citep{von2021self}. These contrastive results separate components only when the assumed positive-pair conditional breaks rotational symmetry. With Gaussian latents the distribution carries no component-wise signal, and separation must come from outside both the conditional and the distribution. A third family keeps i.i.d.\ Gaussian sources and instead restricts the mixing to structured function classes such as conformal maps \citep{buchholz2022function} or volume-preserving, orthogonal-column mixings \citep{gresele2021independent}. Thus, in every case identifiability rests on structure from outside the representation objective, whether an auxiliary variable, a latent conditional built into the likelihood or the loss, or a restriction on the mixing function. Prediction in representation space alone, the setting of the present paper, supplies none of the three, and that is the gap the present theory sets out to fill.

\paragraph{Identifiability from sparse structure.} A separate line obtains identifiability from sparse latent-to-observation structure. Structural sparsity of the mixing Jacobian yields permutation identifiability in nonlinear ICA \citep{zheng2022nonlinearica}, later extended to undercomplete, partially sparse, and grouped settings \citep{zheng2023generalizing}; in the linear Gaussian case, sparsity replaces non-Gaussianity \citep{zhang2009ica,palsson2014sparse,Abrahamsen2018sparse} under the Structural Variability condition \citep{ng2023sparseica}, rooted in support-equivalence characterizations \citep{ghassami2020characterizing}. Sparse mechanisms \citep{lachapelle2021disentanglement}, sparse decoders \citep{moran2022sparsevae}, and additive decoders \citep{lachapelle2024additive} carry the principle to latent dynamics and deep generative models. Each of these anchors recovery through a decoder or an observation likelihood. DSReg needs neither: Structural Diversity plus a faithfulness/no-cancellation condition yields signed-permutation identifiability with no decoder and no likelihood, and the support-pattern requirement itself is strictly weaker than the structural conditions above (Proposition~\ref{prop:diversity-weaker}). Besides, the disentanglement literature pursues individual latent recovery with regularized generative models \citep{higgins2017beta,kim2018disentangling,chen2018isolating,peebles2020hessian} and quantitative evaluation protocols \citep{eastwood2018framework}. The impossibility result of \citet{locatello2019challenging} shows that some signal beyond the marginal distribution is necessary; here that signal is the dependency structure itself rather than any further generative assumption about how the observations were produced.

\paragraph{JEPA world models and the linear premise.} DSReg takes its linear-identifiability premise from the JEPA line of world-model learning. I-JEPA established prediction in representation space as a scalable alternative to pixel reconstruction \citep{assran2023ijepa}, LeJEPA grounded it in an isotropic-Gaussian embedding objective, SIGReg, that is provably optimal for downstream prediction \citep{balestriero2025lejepa}, and JEPA-style world models support planning and control directly in latent space \citep{dinowm2024,leworldmodel2026}. \citet{klindt2026worldmodel} proved that in a Gaussian latent world this recipe identifies the latent state up to an orthogonal transformation. Within that world, the orthogonal class is where identifiability without reconstruction previously stopped. The state is recovered while the individual latents stay mixed by an unknown rotation. DSReg starts from the theorem and resolves the rotation, with dependency sparsity selecting the individual latents and with no reconstruction or likelihood entering at any point of the procedure. 

\paragraph{Sparse codes in JEPAs.} A parallel JEPA line pursues sparsity in the code itself. Rectified LpJEPA replaces the isotropic Gaussian target with a rectified generalized Gaussian, so that representations become sparse and non-negative, with most coordinates zero for a given sample \citep{kuang2026lpjepa}. The two designs answer the same rotational symmetry in complementary ways, and they sparsify different objects. Changing the target law breaks the symmetry distributionally and sparsifies the code; DSReg keeps the Gaussian target and breaks the symmetry structurally, sparsifying the support of the dependency Jacobian $\partial x/\partial\tilde z$, and a dense code can carry a very sparse dependency structure. Sparsity enters DSReg only through what it selects for, never as an assumption about the world: the guarantee asks only that the true footprints be pairwise distinct, which allows them to be almost fully dense in the observed variables while remaining pairwise distinguishable.

\paragraph{Independent mechanism analysis.} Independent mechanism analysis (IMA) \citep{gresele2021independent} also constrains a latent representation through the Jacobian of the observation map. IMA asks the columns of the observation Jacobian $\partial x/\partial z$ to be orthogonal, so that its Gram matrix is diagonal, whereas DSReg asks the columns to have distinct functional supports (Assumption~\ref{asm:structural-diversity}) that do not cancel (Assumption~\ref{asm:functional-no-cancellation}). Neither condition implies the other, since orthogonal columns may share every row and columns with distinct supports need not be orthogonal. Under the premise $h=Qz$, IMA also suggests a post hoc baseline, jointly diagonalizing the estimated Gram matrices $B_a^\top B_a$ across anchors, which pins down the rotation when the diagonal profiles satisfy the uniqueness condition of joint diagonalization, namely that no two of them coincide across the anchors. The distinction between the two approaches is therefore support structure against orthogonality, under a shared linear identifiability premise and a shared estimate of local Jacobians that trains no decoder. The same premise also admits regularizers built on IMA rather than on supports. The IMA contrast of \citet{gresele2021independent} measures how far the Jacobian columns are from orthogonal and is differentiable, so it could take the place of Equation~\eqref{eq:idjepa-objective} on the same local Jacobians. A term penalizing both non-orthogonality and shared support would let the two conditions cover for each other where either holds only approximately. Which of these recovers individual latents, and under what condition on the observation map, is an open question, which we leave to future work, along with the joint diagonalization baseline sketched earlier in this paragraph.

\paragraph{Estimation in practice.} The identifiability statements use exact a.e. dependency supports, with Theorem~\ref{thm:approx-recovery} covering their thresholded perturbation. Equation~\eqref{eq:idjepa-objective} replaces the exact support count by local Jacobian estimates and an anchor-averaged $\ell_1$ relaxation; the surrogate and noise experiments of Appendix~\ref{sec:app-synthetic} (Figures~\ref{fig:optimizer-ablation} and~\ref{fig:structdiv-noise-robustness}) evaluate how faithfully this practical objective tracks its population target.

\FloatBarrier
\section{Additional experiments}
\label{app:experiments}
This appendix expands the empirical part of the paper. Section~\ref{sec:app-metric-defs} fixes the metrics and the protocol shared by every benchmark, and Section~\ref{sec:app-scaling} then opens the experiments with the study readers of Section~\ref{sec:practical-estimator} most often ask for, scaling DSReg far beyond the paper's benchmarks, to $d=8192$ on a single GPU with samples to $10^6$ and observed dimensions to $2^{17}$. Section~\ref{sec:app-synthetic} verifies the regime map and stress-tests the criterion from every side, Section~\ref{sec:app-sparse-use-details} specifies the sparse-use environments and probes, the next three sections cover the learned visual encoders and the two external renderers, and Section~\ref{sec:app-baselines} closes by asking whether any latent-only rotation, or any baseline reading the same signal, could have matched the dependency criterion from the same inputs.

\subsection{Metrics and protocol}
\label{sec:app-metric-defs}

Every benchmark in the paper shares one measurement vocabulary and one experimental protocol. This section fixes both once, so the sections that follow only state what changes.

\paragraph{Research questions.} Every experiment in the paper asks one of three questions: is the latent state present in the representation at all, has it been resolved into individual variables, and can we act through those variables for downstream tasks? One metric family answers each, defined below. Keeping the three separate lets the experiments show that DSReg changes the answers to the second and third while leaving the first answer exactly where LeJEPA had already put it.

\paragraph{Recovery metrics.} Let $Z,\widehat Z\in\mathbb R^{n\times d}$ collect the ground-truth and the estimated latents over $n$ evaluation samples, write $Z_{:,i}$ for a column, and let $\rho(u,v)$ denote the Pearson correlation. The first question is answered by the dense readout
\begin{equation}
\label{eq:metric-dense-r2}
    R^2(h\to z) \;=\; 1-\frac{\min_{W,b}\big\|Z-(\widehat ZW+\mathbf 1b^\top)\big\|_F^2}{\big\|Z-\mathbf 1\bar z^\top\big\|_F^2},
\end{equation}
the coefficient of determination of the best linear map from the whole estimated vector onto all ground-truth latents. Any invertible $A$ acting as $\widehat Z\mapsto\widehat ZA$ leaves it unchanged. A high value therefore settles the first question and is silent on the second.

The second question is answered by the latent mean correlation coefficient, which scores the best one-to-one matching of estimated to true coordinates,
\begin{equation}
\label{eq:metric-mcc}
    \mathrm{MCC} \;=\; \max_{\pi\in S_d}\;\frac1d\sum_{i=1}^{d}\Big|\rho\big(Z_{:,i},\widehat Z_{:,\pi(i)}\big)\Big|,
\end{equation}
maximized over permutations by the Hungarian algorithm for $d\le128$ and by a greedy assignment above it. Besides being sensitive to mixing, which~\eqref{eq:metric-dense-r2} is not, this quantity reaches one precisely when the representation recovers the individual world latents in the sense of Definition~\ref{def:identifiability}, which is why we report it as the main individual-latent recovery score on every benchmark below.

\paragraph{Disentanglement scores on the external renderers.} The renderer benchmarks also report the standard scores, so that the comparison with $\beta$-VAE and $\beta$-TCVAE is made in their own vocabulary. Write $v_1,\ldots,v_K$ for the ground-truth factors and $c_1,\ldots,c_d$ for the estimated coordinates. DCI \citep{eastwood2018framework} fits gradient-boosted trees predicting each factor from all coordinates and collects the absolute feature importances into $R\in\mathbb R^{d\times K}$. It then measures how far each coordinate concentrates on a single factor,
\begin{equation}
\label{eq:metric-dci}
    \mathrm{DCI}=\sum_{i=1}^{d}\rho_i\Big(1-H_K\big(P_{i,:}\big)\Big),\qquad P_{ij}=\frac{R_{ij}}{\sum_{k}R_{ik}},\qquad \rho_i=\frac{\sum_{j}R_{ij}}{\sum_{i',j}R_{i'j}},
\end{equation}
with $H_K$ the entropy in base $K$. The informativeness column of Table~\ref{tab:shapes3d-gaussian} is the held-out $R^2$ of those same regressions. MIG \citep{chen2018isolating} discretizes coordinates and factors into twenty uniform-width bins, then averages the normalized gap between the two most informative coordinates for each factor. SAP \citep{kumar2017variational} replaces mutual information by the squared correlation of the univariate fit,
\begin{equation}
\label{eq:metric-mig-sap}
    \mathrm{MIG}=\frac1K\sum_{k=1}^{K}\frac{I(v_k;c_{(1)})-I(v_k;c_{(2)})}{H(v_k)},\qquad \mathrm{SAP}=\frac1K\sum_{k=1}^{K}\Big(\rho^2\big(v_k,c_{(1)}\big)-\rho^2\big(v_k,c_{(2)}\big)\Big),
\end{equation}
where $c_{(1)}$ and $c_{(2)}$ are the two coordinates ranked highest for $v_k$ by the score each metric uses.

\paragraph{Sparse-use probes.} The third question has no closed form. What matters is whether a downstream module restricted to a few coordinates can do its job, so each probe is a task with its own success criterion. The four are visual editing, sparse model predictive control, five-step rollout prediction, and the detection of physically implausible next states. Section~\ref{sec:app-sparse-use-details} gives the environments and the exact construction of each, and the readout probe that recurs across benchmarks is described next.

\paragraph{Design principle.} One principle governs every experiment. Ground-truth latents define benchmarks and evaluation scores, never the method. The dependency rotation is always fit from estimated latents and observations, or a fixed observation preprocessing chosen before fitting; labels enter only through benchmark construction, small-label evaluation probes, and metrics. The benchmarks differ only in where the representation comes from: exact orbits $h=Qz$ isolate the rotation step under known structure, encoders trained from pixels test whether the gain survives learning, and external renderers remove the observation map from our control. Each step takes one more piece of our own design out of the loop, and the gain survives each removal, as the sections below elaborate.

\paragraph{Few-shot single-latent readout.} This recurring probe appears on several benchmarks and deserves a procedural description. For each ground-truth factor, a handful of labeled examples select the single estimated coordinate that correlates best with the factor and fit a scalar affine map from that coordinate alone; the probe is forbidden to mix coordinates, so its score is the price of naming the variables. On a representation whose coordinates are individual factors, a few labels suffice and the readout approaches the dense ceiling; on a mixture no scalar map works at any label budget, which is why this probe separates the two representations more sharply than any dense score.

\paragraph{Optimization.} The dependency step regresses local Jacobians over nearest neighbors (ridge $10^{-3}$; $k=32$ on the small estimated-Jacobian benchmarks, raised where the latent dimension requires it, since the neighborhood must exceed $d$) and optimizes the rotation by Adam on the matrix-exponential parameterization of $O(N)$; at larger dimension the exponential gives way to the Cayley chart of Appendix~\ref{sec:app-scaling}. Anchor counts, step budgets, and restarts grow with the dimension of the benchmark. These choices carry no load. Figure~\ref{fig:optimizer-ablation} varies the sparsity surrogate and the anchor budget directly and finds recovery insensitive to both, so we report defaults rather than tuned values.

\subsection{Scaling the dependency criterion}
\label{sec:app-scaling}

Section~\ref{sec:practical-estimator} claims that nothing in DSReg is bound to the small worlds of the benchmarks. This section backs the claim: the factorized algebra first, then one axis at a time.

\paragraph{The Jacobians are never materialized.} The DSReg criterion of Section~\ref{sec:practical-estimator} averages the rotated local Jacobians over $m$ anchors, and written as Equation~\eqref{eq:idjepa-objective} it appears to require $m$ matrices of size $p\times d$, which is the binding cost once the observed dimension $p$ or the latent dimension $d$ grows. That reading is pessimistic. The ridge estimate at anchor $a$ is built from the $k$ neighbors of $h_a$, and if $\Delta H_a\in\mathbb R^{k\times d}$ and $\Delta X_a\in\mathbb R^{k\times p}$ stack the neighbor differences in representation and observation space, then
\begin{equation}
\begin{aligned}
B_a &\;=\; \Delta X_a^\top \Delta H_a\left(\Delta H_a^\top \Delta H_a + \lambda I\right)^{-1} \;=\; \Delta X_a^\top P_a , \\
P_a &\;=\; \Delta H_a\left(\Delta H_a^\top \Delta H_a + \lambda I\right)^{-1} \;\in\; \mathbb R^{k\times d} .
\end{aligned}
\label{eq:dsreg-factorization}
\end{equation}
Two consequences follow immediately. First, $B_a$ factors through the $k$ dimensional neighborhood, so $\operatorname{rank}(B_a)\le\min(k,d)$, and since the estimator needs $k>d$ the rank is at most $d$: whenever $p$ exceeds the latent dimension, the $p\times d$ array encodes an object with far fewer degrees of freedom than its size suggests. Second, $B_a$ is determined by the raw observations, which are already in memory as the dataset, together with the neighbor indices and the factor $P_a$. Storing $\{(a,\mathcal N_a,P_a)\}_{a=1}^m$, which is Algorithm~\ref{alg:dsreg-scale}(a), costs $O(mkd)$ and replaces the $O(mpd)$ of the materialized form, a reduction by $p/k$, which is close to two orders of magnitude at the largest observed dimension we run.

\paragraph{The exact objective streams in two nested blocks.} Nothing about the factorization is an approximation, and the criterion can be evaluated from it exactly, as Algorithm~\ref{alg:dsreg-scale}(c) does. The only requirement is to associate the products in the right order: forming $W_a=P_aR^\top$ first, at cost $O(kd^2)$, and then $\Delta X_a^\top W_a$, at cost $O(kpd)$, avoids ever building $B_a$, whereas the natural left-to-right order pays $O(pd^2)$ per anchor and allocates the $p\times d$ result. The subgradient comes out of the same quantities. Writing $M_a=B_aR^\top$ for the rotated Jacobian,
\begin{equation}
\nabla_R\,\mathcal L_{\mathrm{DSReg}}(R) \;=\; \frac{1}{m}\sum_{a=1}^{m}\operatorname{sign}(M_a)^\top B_a \;=\; \frac{1}{m}\sum_{a=1}^{m}\Big(\operatorname{sign}(M_a)^\top \Delta X_a^\top\Big) P_a ,
\label{eq:dsreg-factored-grad}
\end{equation}
where the parenthesized factor is $d\times k$ and the product with $P_a$ is $d\times d$, so the gradient is assembled at the same cost as the value and again never touches a $p\times d$ array. In the implementation the anchors are traversed in chunks and, inside each chunk, the observed coordinates in blocks, so peak memory is set by the two block sizes rather than by $m$ or $p$, and both can be tuned freely to the device at hand without changing the value that the pass computes in exact arithmetic.

\paragraph{An unbiased estimator makes the cost independent of $p$ and $m$.} Streaming keeps memory bounded but still reads every anchor and every observed coordinate at each step. Because the criterion is an average, it also admits an unbiased stochastic estimate, which is Algorithm~\ref{alg:dsreg-scale}(d). Normalizing Equation~\eqref{eq:idjepa-objective} by the number of entries, the objective is the expectation
\begin{equation}
\mathcal L_{\mathrm{DSReg}}(R) \;\propto\; \mathbb E_{(a,r)}\!\left[\tfrac{1}{d}\big\|R\,b_{a,r}\big\|_1\right],
\qquad b_{a,r} \;=\; P_a^\top \Delta X_a[:,r] \;\in\; \mathbb R^{d},
\label{eq:dsreg-sampled}
\end{equation}
over a uniformly drawn anchor-coordinate pair $(a,r)$, since $b_{a,r}$ is exactly the $r$-th row of $B_a$ and $(B_aR^\top)_{r,:}=R\,b_{a,r}$. Averaging Equation~\eqref{eq:dsreg-sampled} over $s$ independent draws is therefore unbiased for the full criterion, and its subgradient $\operatorname{sign}(Rb_{a,r})\,b_{a,r}^\top$ is rank one in $R$. A step touches $O(sd)$ memory for the sampled rows and reads one anchor's factor at a time, so per-step cost no longer depends on $p$ or $m$ at all. Drawing the pairs grouped by anchor matters in practice: a flat gather of the factor tensor would materialize an $s\times k\times d$ copy, which at large $k$ and $d$ is tens of gigabytes per step.

\begin{algorithm}[t]
\footnotesize
\caption{\textbf{DSReg at scale.} All three evaluation modes consume the same factors from part~(a) and return the same objective, differing only in what they hold in memory. Annotations on the right give the shape of the array formed on that line, and its cost where that is the point. Here $n$ is the sample count, $p$ the observed dimension, $d$ the latent dimension, $m$ the anchor count and $k>d$ the neighborhood size, and all three modes return the value together with its subgradient.}
\label{alg:dsreg-scale}
\vspace{2pt}
\hrule
\vspace{3pt}
{\raggedright\textbf{(a) Neighborhood factors.} Run once, shared by (b), (c) and (d).\par}
\begin{algorithmic}[1]
\REQUIRE $h\in\mathbb R^{n\times d}$, $x\in\mathbb R^{n\times p}$, anchors $m$, neighbors $k>d$, ridge $\lambda$
\STATE sample anchor indices $\mathcal A\subset[n]$ with $|\mathcal A|=m$
\FOR{$a\in\mathcal A$}
\STATE $\mathcal N_a\leftarrow$ the $k$ nearest neighbors of $h_a$ in $h$; \quad $\Delta H_a\leftarrow h[\mathcal N_a]-h_a$ \hfill shape $k\times d$
\STATE $P_a\leftarrow \Delta H_a(\Delta H_a^\top\Delta H_a+\lambda I)^{-1}$ \hfill shape $k\times d$
\ENDFOR
\RETURN $\{(a,\mathcal N_a,P_a)\}_{a\in\mathcal A}$ \hfill storage $O(mkd)$, not $O(mpd)$
\end{algorithmic}
\vspace{3pt}
\hrule
\vspace{3pt}
{\raggedright\textbf{(b) Materialized.} Fastest while it fits; forms the $p\times d$ array explicitly.\par}
\begin{algorithmic}[1]
\REQUIRE factors, $x$, rotation $R$; \quad $L\leftarrow0$, $G\leftarrow0_{d\times d}$
\FOR{each anchor $a$}
\STATE $\Delta X_a\leftarrow x[\mathcal N_a]-x_a$ \hfill shape $k\times p$
\STATE $B_a\leftarrow \Delta X_a^\top P_a$ \hfill shape $p\times d$; (c) and (d) avoid this
\STATE $M_a\leftarrow B_aR^\top$; \quad $L\mathrel{+}=\textstyle\sum_{r,j}|(M_a)_{rj}|$; \quad $G\mathrel{+}=\operatorname{sign}(M_a)^\top B_a$ \hfill cost $O(pd^2)$
\ENDFOR
\RETURN $L/(mpd)$, \; $G/(mpd)$
\end{algorithmic}
\vspace{3pt}
\hrule
\vspace{3pt}
{\raggedright\textbf{(c) Factored.} Same value as (b), about half the memory, more time.\par}
\begin{algorithmic}[1]
\REQUIRE factors, $x$, rotation $R$, anchor chunk size, feature block size; \quad $L\leftarrow0$, $G\leftarrow0_{d\times d}$
\FOR{each chunk of anchors}
\STATE $W_a\leftarrow P_aR^\top$ for $a$ in the chunk \hfill shape $k\times d$, cost $O(kd^2)$
\FOR{each block $F\subset[p]$ of observed coordinates}
\STATE $\Delta X_a^{F}\leftarrow x[\mathcal N_a,F]-x_a[F]$ \hfill shape $k\times|F|$, recomputed each pass
\STATE $M\leftarrow (\Delta X_a^{F})^\top W_a$ \hfill shape $|F|\times d$, cost $O(k|F|d)$
\STATE $L\mathrel{+}=\sum|M|$; \quad $G\mathrel{+}=\big(\operatorname{sign}(M)^\top(\Delta X_a^{F})^\top\big)P_a$ \hfill inner factor $d\times k$
\ENDFOR
\ENDFOR
\RETURN $L/(mpd)$, \; $G/(mpd)$
\end{algorithmic}
\vspace{3pt}
\hrule
\vspace{3pt}
{\raggedright\textbf{(d) Row-sampled.} Unbiased; per-step cost independent of $p$ and $m$.\par}
\begin{algorithmic}[1]
\REQUIRE factors, $x$, rotation $R$, rows per step $s$
\STATE draw $(a_1,r_1),\ldots,(a_s,r_s)$ independently and uniformly from $[m]\times[p]$
\STATE group the draws by anchor \hfill a flat gather would allocate $s\times k\times d$
\FOR{each distinct anchor $a$ in the draw, with its sampled columns $F_a$}
\STATE $\Delta X_a^{F_a}\leftarrow x[\mathcal N_a,F_a]-x_a[F_a]$; \quad $b_{a,r}\leftarrow P_a^\top\Delta X_a[:,r]$ for $r\in F_a$ \hfill rows of $B_a$, shape $d$
\ENDFOR
\RETURN $\widehat L=\frac{1}{sd}\sum_{t}\|Rb_{a_t,r_t}\|_1$, \; $\widehat G=\frac{1}{s}\sum_{t}\operatorname{sign}(Rb_{a_t,r_t})\,b_{a_t,r_t}^\top$ \hfill $\mathbb E[\widehat L]\propto\mathcal L_{\mathrm{DSReg}}(R)$
\end{algorithmic}
\end{algorithm}

\paragraph{At large $d$ the rotation chart matters more than the criterion.} Once the criterion is cheap, the remaining cost is the orthogonality constraint. Optimizing over $O(d)$ with a gradient method requires writing $R$ in terms of unconstrained parameters. We use the matrix exponential $R=\exp(A-A^\top)R_0$ throughout the paper, which is the standard choice and is well behaved up to a few thousand dimensions. Its autograd workspace is a $2d\times2d$ block exponential, and beyond $d\approx16000$ that workspace, rather than the Jacobians, becomes the binding memory cost. The Cayley transform $R=(I+S)^{-1}(I-S)R_0$ with $S=A-A^\top$ is an alternative chart on the same manifold, orthogonal for every $A$ because $S$ is skew-symmetric, and it needs one linear solve in place of the exponential. It is what we use at the largest dimension; low-rank orthogonal updates \citep{kiani2022projunn} are the natural next step for anyone who needs to go beyond the dimensions we report.

\paragraph{Which mode to use.} Parts~(b),~(c) and~(d) of Algorithm~\ref{alg:dsreg-scale} compute the same objective from the same factors and differ only in what they hold in memory, so the choice among them is purely one of cost. Materialize while the Jacobians still fit on the device, since below $d\approx512$ the dense product is also the fastest of the three modes (Table~\ref{tab:app-scaling-dim}). Use the factored form when they stop fitting but the objective must stay exact. It is the memory-frugal mode rather than the fast one: it roughly halves peak memory at $d=2048$ and pays for that in step time, because the neighbor differences are recomputed on every pass instead of being held as a $p\times d$ array. Use row sampling when $p$ or the anchor set is large, since it is the only mode whose per-step cost is flat in both and the only one that still runs once the materialized route no longer fits the card, and accept that the objective is then estimated rather than computed exactly at every step. One requirement survives every mode and is not negotiable: the neighborhood must satisfy $k>d$, since $k\le d$ makes $\Delta H_a$ rank deficient and the recovered rotation degenerates, with $k=1.5d$ the default we recommend and $k=1.25d$ the smallest ratio that still worked at every scale we tested (Figure~\ref{fig:dsreg-scaling}b), a requirement no mode can rescue.

\paragraph{Setup.}\looseness=-1 The testbed is a banded synthetic world with known footprints: each observed coordinate depends on two adjacent latents through distinct nonlinear components, footprints are length-four circular windows at distinct offsets with $p=2d$, and the representation is a random orthogonal mixing of the true latents, so recovery is scored by MCC at any scale. Unless stated otherwise the full estimated procedure runs with $k=1.5d$ neighbors, $256$ anchors, and three seeds on one $48$\,GB GPU. Table~\ref{tab:app-scaling-dim} scales the latent dimension, and Figure~\ref{fig:dsreg-scaling} summarizes recovery and cost across every axis of the study.

\begin{table}[t]
\centering
\small \caption{\textbf{Scaling across latent dimension.} The full estimated procedure, three seeds each. At $d\geq 1024$ the sampled columns use a compute-matched budget whose wall-clock stays at or below the materialized one; at smaller $d$ the standard budget already suffices, and sampling holds no advantage there since its per-step cost is nearly $d$-independent while the materialized cost is what grows. All three modes optimize the same objective, and the materialized and factored columns agree on MCC to within $0.002$ everywhere, which is a live check that the factorization is exact. The factored mode trades time for memory, roughly halving peak memory at $d=2048$ at about $2.8\times$ the step time, since it recomputes the neighbor differences on each pass. Axes without a panel of their own, each measured on its own sweep with its own baseline rather than on the rows above: raising the sample count to $10^6$ leaves recovery and step time flat to within seed noise ($0.826\to0.833$ at $d=1024$, $49$\,ms throughout; $0.822$ at $d=2048$), and growing the observed dimension to $2^{17}$ holds the sampled step at $47$\,ms while MCC improves from $0.806$ to $0.884$ across the full doubling sweep.}
\label{tab:app-scaling-dim}
\setlength{\tabcolsep}{4pt}
\begin{tabular}{lccccccccccc}
\toprule
& & \multicolumn{3}{c}{MCC} & \multicolumn{3}{c}{step time (ms)} & \multicolumn{3}{c}{peak memory (GB)} \\
\cmidrule(lr){3-5}\cmidrule(lr){6-8}\cmidrule(lr){9-11}
$d$ & $n$ & mat. & fact. & samp. & mat. & fact. & samp. & mat. & fact. & samp. \\
\midrule
128  & $20$k   & $0.853$ & $0.853$ & $0.850$ & $7.5$   & $4.9$    & $42.6$ & $0.14$  & $0.11$  & $0.11$ \\
256  & $20$k   & $0.843$ & $0.844$ & $0.840$ & $7.4$   & $6.8$    & $43.9$ & $0.47$  & $0.31$  & $0.30$ \\
512  & $24.6$k & $0.833$ & $0.831$ & $0.829$ & $13.1$  & $34.0$   & $43.8$ & $1.76$  & $1.07$  & $1.02$ \\
1024 & $49.2$k & $0.827$ & $0.828$ & $0.823$ & $101.6$ & $231.6$  & $53.5$ & $6.97$  & $4.20$  & $3.26$ \\
2048 & $98.3$k & $0.821$ & $0.820$ & $0.809$ & $699.8$ & $1933.2$ & $92.5$ & $27.76$ & $13.73$ & $11.42$ \\
\bottomrule
\end{tabular}
\end{table}

\begin{figure}[t]
    \centering
    \includegraphics[width=\linewidth]{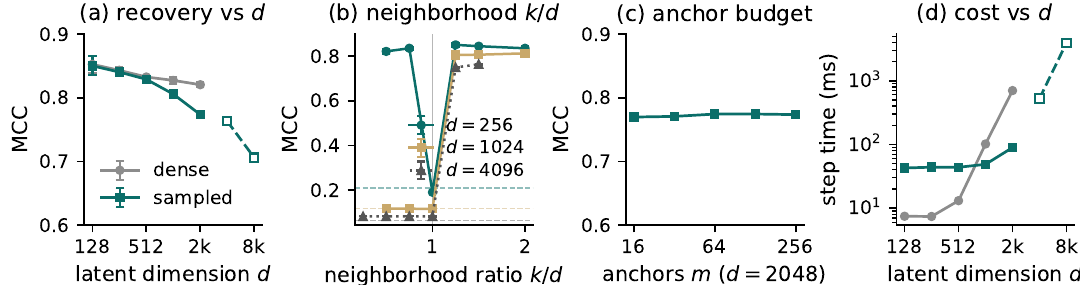}
    \caption{\textbf{Scaling the full estimated procedure.} (a)~Recovery declines gently with $d$; open markers are the single-GPU frontier, where anchor count shrinks to fit memory. (b)~Recovery against the neighborhood ratio $k/d$ at three scales; dashed lines mark the analytic score of an unrotated mixing, $\sqrt{2\ln d/d}$. (c)~Sixteen anchors already match $256$ at $d=2048$. (d)~The sampled step cost is nearly flat in $d$ while the materialized cost grows two orders of magnitude, with the curves crossing between $d=512$ and $d=1024$.}
    \label{fig:dsreg-scaling}
\end{figure}

\paragraph{One GPU carries the full method to $d=8192$.}\looseness=-1 Recovery declines gently with $d$ while the sampled step cost stays nearly flat against a steep rise for the materialized evaluation, at consistently lower memory (Table~\ref{tab:app-scaling-dim}). Past the point where the materialized route stops fitting on the card, row sampling is the only remaining option, and it carries recovery to $d=8192$ within the card's memory.

\paragraph{No other axis is the bottleneck.}\looseness=-1 The remaining cost axes are flat in the way the algebra predicts (Figure~\ref{fig:dsreg-scaling}). Raising the sample count to $10^6$ leaves recovery and per-step cost unchanged, and doubling the observed dimension up to $2^{17}$ leaves the sampled step cost flat while recovery improves, since each latent leaves its footprint on proportionally more observed coordinates. Moreover, the criterion is anchor-efficient: a small fraction of the default anchor budget already recovers as well as the full budget, which is what makes the frontier's reduced anchor counts benign, and a protocol-matched control confirms the leaner protocol costs almost nothing. Meanwhile, the decline with $d$ is independent of the footprint family, with nested and random-overlap constructions tracking the banded family at every dimension tested, with no family-specific tuning.

\paragraph{The neighborhood must span the latent dimension.}\looseness=-1 Scaling the estimator exposes one requirement of its own, with a margin (Figure~\ref{fig:dsreg-scaling}b). Neighborhoods with $k<d$ make $B_a$ rank-deficient by construction; at small $d$ this only dents recovery, but the surviving anisotropy washes out with dimension, and at larger $d$ the criterion loses the rotation signal entirely, with the optimizer returning its starting point and the measured score matching an unrotated random orthogonal mixing. Meanwhile, square neighborhoods $k=d$ fail at every scale for a different reason, since the neighborhood Gram matrix then sits at the hard edge of its spectrum and the ridge solve amplifies noise. Recovery is restored from $k=1.25d$ at every scale tested, which motivates the $k=1.5d$ default and its margin. The practical frontier on one $48$\,GB card is therefore $d=8192$, set by the memory a well-posed local regression needs rather than by the optimizer. The next rung has a measured price: at $d=16384$ the well-posed construction needs an $80$\,GB device (with a Cayley parameterization replacing the matrix exponential, whose autograd workspace dominates at this scale), and there the binding constraint shifts from memory to the compute budget available to the rotation search.

\paragraph{Shortcuts that discard structure fail.}\looseness=-1 Two natural approximations fail, for reasons the theory predicts. Orthogonally invariant sketches of the Jacobian, such as Hutchinson-style estimates of $\lVert B_a R^\top\rVert_F$, are blind to $R$ by construction, so nothing remains to optimize. Dense random projections of the observed coordinates destroy the footprints the criterion reads, since each projected observed coordinate mixes rows across different footprints; only locality-preserving reductions such as the row subsampling used above keep the footprints that the criterion needs intact.
\subsection{Verification and robustness}
\label{sec:app-synthetic}

This subsection verifies the synthetic evidence behind Section~\ref{sec:exp-synthetic} and then stress-tests it from every side: the regime construction itself, the end-to-end procedure, the optimizer and the training schedule, corrupted dependency signals, the stability guarantee of Theorem~\ref{thm:approx-recovery}, and the boundary of Assumption~\ref{asm:functional-no-cancellation}. Each block below states its conclusion in its own heading.

\paragraph{Recovery follows Structural Diversity at every dimension.} The regime map of Figure~\ref{fig:structdiv-regimes} carries weight only if each panel realizes the regime the theory names, so the sweep begins by verifying the conditions on the support masks themselves: the diverse pattern satisfies the prior certificates, the nested pattern violates Non-Inclusion and both structural-sparsity certificates, the minimal-difference pattern also violates Structural Variability, and the identical pattern violates Structural Diversity itself. Figure~\ref{fig:regime-scaling} then reports the full sweep from $N=8$ to $N=128$, five runs per point, on analytic orbits with nonlinear observed variables and exact anchor-wise dependency Jacobians, so estimation error plays no role in this comparison. Recovery stays at ceiling in every regime that satisfies Structural Diversity while the unrotated baseline decays steadily with dimension. Meanwhile, the identical regime sits at its shared-subspace level, which is the best that support information allows, and the pair span is still recovered. Each orbit draws a generator whose Jacobian support equals the prescribed footprint mask, with random nonlinear components on the active entries, so each panel's regime is enforced by construction and verified on the support masks before any run is scored.

\paragraph{The end-to-end procedure uses no oracle information.} The learned-encoder benchmark of Figure~\ref{fig:synthetic-mcc} runs the whole method end to end on the same family. An encoder is trained from scratch on consecutive-state pairs of the MLP world, its representation is whitened and frozen, and the rotation is fit from observations and estimated latents with the default settings, with recovery scored on held-out data. Nothing about the world, the footprints, or the latents enters training, so this benchmark measures the procedure a practitioner would run on new data.

\begin{figure}[tbp]
    \centering
    \includegraphics[width=0.52\linewidth]{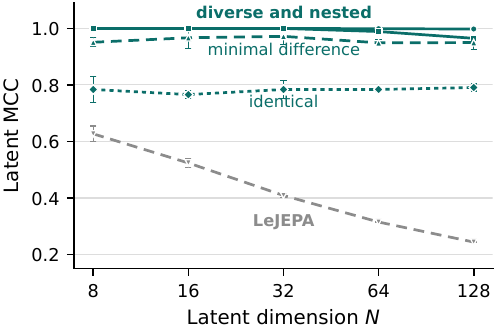}
    \caption{\textbf{Footprint regimes across dimension.} Recovery follows Structural Diversity at every dimension: diverse, nested, and minimal-difference footprints stay at ceiling while identical footprints sit at the shared-subspace level, far above the decaying baseline of the unrotated representation; the identical pair's span is still recovered, with CCA at least $0.998$ at every dimension of the full regime sweep.}
    \label{fig:regime-scaling}
\end{figure}

\paragraph{Optimizer choices carry no load.} Figure~\ref{fig:optimizer-ablation} swaps the sparsity surrogate and varies the anchor budget on the $N=16$ orbits, five runs each; each panel changes one choice while the other stays at the default ($\ell_1$, $128$ anchors), and every surrogate starts from the same six rotations. Recovery is insensitive to the surrogate on diverse footprints, and larger anchor budgets help exactly where the theory says the evidence is scarcest. In the minimal-difference regime a single observation row separates two footprints, and finding it takes more anchors. The two rightmost bars test the relaxation against the criterion it relaxes. A hard support count, thresholded at $\tau=0.05$ as in Theorem~\ref{thm:approx-recovery} and minimized by coordinate search over Givens rotations, recovers at least as well as the $\ell_1$ default at roughly twice the cost, and where a gap to exact recovery remains, as in the minimal-difference regime, refining the $\ell_1$ solution by the same count closes it within seconds.

\begin{figure}[tbp]
    \centering
    \includegraphics[width=0.78\linewidth]{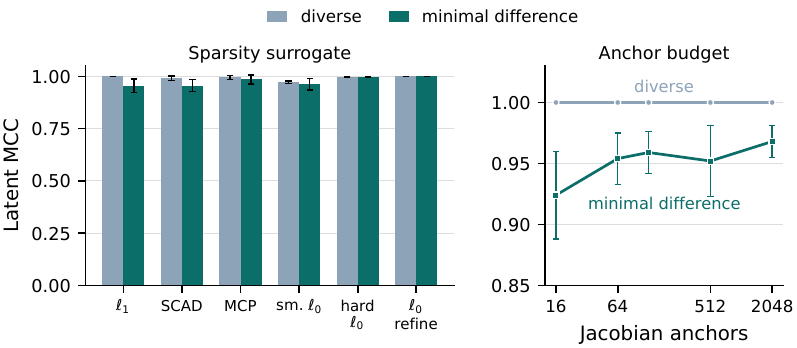}
    \caption{\textbf{Recovery is insensitive to the optimizer's choices.} Recovery is insensitive to the sparsity surrogate and to the anchor budget; only the minimal-difference regime rewards more anchors, exactly where the evidence is scarcest. The two rightmost bars replace the relaxation by a hard support count at $\tau=0.05$, minimized by coordinate search from the same six starts, and by that count used to refine the $\ell_1$ solution afterwards. Both series are DSReg selections, with color denoting the underlying footprint regime.}
    \label{fig:optimizer-ablation}
\end{figure}

\paragraph{The training schedule is immaterial.} Table~\ref{tab:joint-ablation} compares fitting the rotation after training (the paper's default), jointly with it, and jointly followed by the decoupled fit, on Gaussian 3DShapes with twenty runs per arm on a shared encoder trajectory. The joint head tracks the decoupled solution and the final fit closes the remaining gap, so the schedule is purely a matter of convenience.

\begin{table}[H]
\centering
\small \caption{\textbf{Training-schedule ablation on Gaussian 3DShapes.} Latent MCC (mean~$\pm$~std), twenty runs per arm on a shared encoder trajectory; the rotation head never feeds back into the encoder.}
\label{tab:joint-ablation}
\begin{tabular}{lc}
\toprule
Schedule & Latent MCC \\
\midrule
LeJEPA (no rotation) & $0.640\pm0.041$ \\
Joint & $0.906\pm0.014$ \\
Joint, then decoupled fit & $0.922\pm0.005$ \\
Decoupled (default) & $0.921\pm0.005$ \\
\bottomrule
\end{tabular}
\end{table}

\paragraph{Recovery degrades gracefully under corrupted Jacobians.} Figure~\ref{fig:structdiv-noise-robustness} corrupts the dependency Jacobian consumed by the optimizer with Gaussian noise, on fresh analytic orbits at $N=16$ with $h=Qz$ and five runs per level. The observation rows carry the same nonlinear channels as the rest of the appendix, so Assumption~\ref{asm:functional-no-cancellation} holds and the noiseless cell is a genuine signed permutation rather than a point on the excluded stratum of Figure~\ref{fig:linear-stratum}. Recovery stays at ceiling across the whole noise range, and the identical-footprint control stays at its shared-subspace level throughout, so the corruption degrades the two regimes by the same small amount without moving either off its own ceiling. The behavior is what one expects from an estimator of a stable population criterion.

\begin{figure}[tbp]
    \centering
    \includegraphics[width=0.52\linewidth]{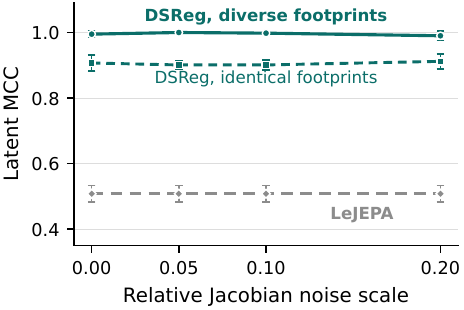}
    \caption{\textbf{Corrupted Jacobians leave recovery at ceiling.} Diverse footprints stay above $0.99$ through noise $0.20$, and identical footprints stay at their predicted shared-subspace level. Nonlinear observation rows, so Assumption~\ref{asm:functional-no-cancellation} holds; $N=16$, five seeds. LeJEPA applies no rotation, so its curve is flat by construction.}
    \label{fig:structdiv-noise-robustness}
\end{figure}

\paragraph{Dense readouts are unchanged by the rotation.} A further check returns to the MLP benchmark of Section~\ref{sec:exp-synthetic}. Figure~\ref{fig:mlp-control} compares the two methods at $N=8$. Dense recovery is identical for the two methods, while individual-latent recovery, few-shot readout from eight labels, and sparse edit error move decisively in DSReg's favor. Notably, the few-shot readout swings from below zero, worse than predicting the mean, to within a hair of the dense readout ceiling.

\begin{figure}[tbp]
    \centering
    \includegraphics[width=0.72\linewidth]{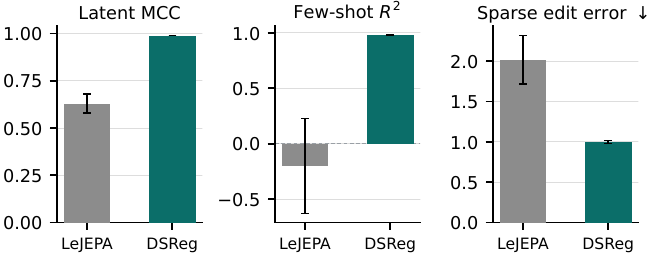}
    \caption{\textbf{Sparse access improves while dense content is unchanged.} At identical dense recovery ($R^2=0.988$ for both methods), the rotation transforms individual-latent recovery, few-shot readout from eight labels, and sparse editing, each moving decisively in DSReg's favor on the MLP benchmark at $N=8$.}
    \label{fig:mlp-control}
\end{figure}

\paragraph{The stability guarantee holds with room to spare.} Theorem~\ref{thm:approx-recovery} guarantees recovery whenever $\tau+\varepsilon<\sigma\rho^*(\delta)$, and a dedicated sweep instantiates every quantity in that inequality on analytic orbits with nonlinear observation rows and verified distinct footprints ($N\in\{4,8,16\}$, five seeds), measuring $\sigma$ and $\rho^*(\delta)$ directly from the data, corrupting the Jacobians at magnitude $\varepsilon$, and running both the thresholded criterion and the $\ell_1$ implementation over a grid of $(\tau,\varepsilon)$ cells. Every cell satisfying the guarantee succeeds, for both estimators, at every dimension, and success extends well beyond the guaranteed region, so the bound is sufficient rather than tight (Figure~\ref{fig:theorem3-phasemap}). A control with linear observation rows trips the Assumption~\ref{asm:functional-no-cancellation} check, confirming that the verification is not vacuous.

\begin{figure}[tbp]
    \centering
    \includegraphics[width=\linewidth]{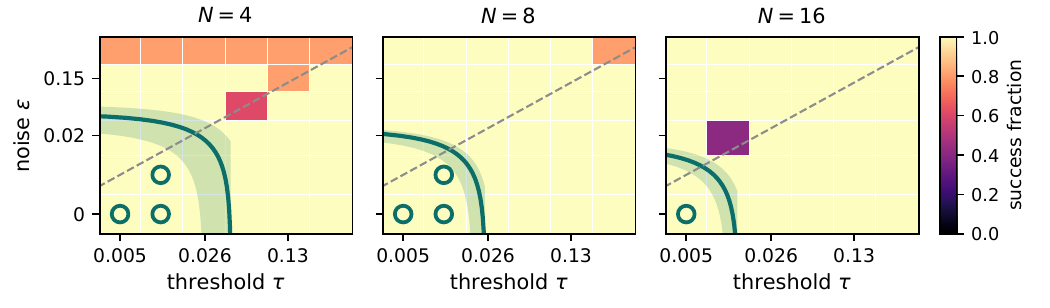}
    \caption{\textbf{Success strictly contains the guarantee.} Success fraction of the thresholded criterion over the $(\tau,\varepsilon)$ grid, five seeds. Solid curve: the guarantee boundary $\tau+\varepsilon=\sigma\rho^*(\delta^*)$ of Theorem~\ref{thm:approx-recovery}, with seed-range shading; circled cells satisfy the guarantee for every seed, and all of them succeed. The $\ell_1$ rotation used throughout the paper reproduces this map on all but four of the $108$ cells, again succeeding on every guaranteed cell.}
    \label{fig:theorem3-phasemap}
\end{figure}

\paragraph{A small nonlinear component already restores recovery.} Assumption~\ref{asm:functional-no-cancellation} excludes, by design, worlds whose observation rows are all exactly linear in their active latents, and a second sweep measures how thin that excluded boundary is. Each row interpolates between a linear and a nonlinear function, $x_r=(1-\gamma)\,\ell_r+\gamma f_r$, so $\gamma=0$ places every row on the excluded linear stratum; the measured margin $\sigma$ grows from exactly $0$ with $\gamma$, and recovery follows it. On the stratum the criterion returns a mixed rotation rather than a signed permutation, and recovery is restored from $\gamma^*\!\approx\!0.1$ at every dimension tested ($N\in\{8,16,32\}$, ten seeds, exact and estimated Jacobians in close agreement), identically in the diverse and nested regimes (Figure~\ref{fig:linear-stratum}).

\begin{figure}[tbp]
    \centering
    \includegraphics[width=\linewidth]{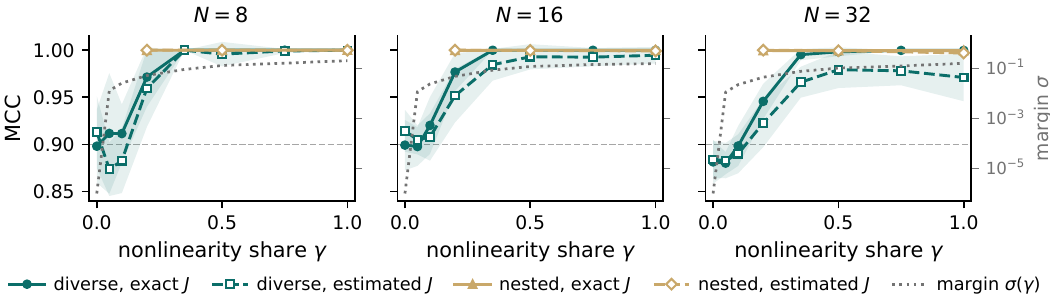}
    \caption{\textbf{The excluded linear stratum is thin.} MCC against the nonlinearity share $\gamma$ with the measured no-cancellation margin $\sigma$ overlaid (dotted, right axis). Where Assumption~\ref{asm:functional-no-cancellation} fails ($\gamma=0$, $\sigma=0$) the criterion returns a mixed rotation; recovery crosses the $0.9$ level (dashed) at $\gamma^*\approx0.1$ and reaches the ceiling by $\gamma\approx0.35$ at every dimension, in both regimes, with exact and estimated Jacobians in close agreement.}
    \label{fig:linear-stratum}
\end{figure}

\subsection{Sparse-use probes: environments and procedures}
\label{sec:app-sparse-use-details}

This section specifies the environments and probe constructions behind Figure~\ref{fig:sparse-use-summary}, in enough detail to reproduce them and to check that no probe smuggles in oracle information.

\paragraph{Every probe sees the full state through a random rotation.} The probes of Section~\ref{sec:exp-sparse-use} run in three environments, TwoRoom (navigation between two rooms), PushT (block pushing), and FetchSlide (robot-arm puck sliding); a scripted waypoint policy with small action noise collects trajectories rendered at $64\times64$. The representation is the analytic orbit $h=Qz$ of Figure~\ref{fig:sparse-use-summary}, the standardized environment state mixed by a random orthogonal $Q$ drawn per run. The dependency rotation is fit as everywhere else in the paper: a fixed pooled-pixel feature map chosen before fitting, nearest-neighbor ridge Jacobians at anchor points, and the anchor-averaged $\ell_1$ criterion of Section~\ref{sec:practical-estimator} under the default dependency settings of Section~\ref{sec:app-metric-defs}, with no benchmark-specific tuning.

\paragraph{Editing succeeds only when one coordinate moves one factor.} Start and goal frames are chosen so that the goal moves one ground-truth factor substantially while leaving the others nearly fixed (two factors in the compositional PushT variant). The planner selects the estimated coordinate with the largest start-to-goal change, interpolates only that coordinate (the two largest in the compositional variant), and retrieves the nearest dataset frame at each step. An edit succeeds when the retrieved endpoint reaches the goal in ground-truth state space while the off-target factors stay close to their start values. A mixed representation fails structurally: one coordinate drags several factors.

\paragraph{A capped sparse model is well specified only in the physical basis.} The control, rollout, and surprise probes share one protocol. Train and test episodes are disjoint, and from $64$ logged training transitions the module fits a linear model of the representation change: the candidate inputs are the estimated latents and the actions, each output coordinate keeps only its top few inputs by absolute correlation (two for control and monitoring, one for rollout) for a ridge regression on those alone, and a separate linear readout maps representations to states for scoring goals and endpoints in state terms. In the physical basis the dynamics are approximately sparse and the capped model is close to well specified; in a rotated basis every input and output mixes all physical variables and the same capped model becomes systematically misspecified rather than merely noisy under the same cap.

\paragraph{Control skill lies entirely in ranking real action sequences.} Each query pairs a start frame with the state actually reached eight steps later, far enough to make planning nontrivial. The planner, given no policy and no gradients, ranks eight-step action subsequences replayed from the training episodes, each rolled through the fitted model from the start representation; the one whose predicted endpoint lands closest to the goal is executed in the simulator from the true start state, and success means reaching the goal within a fixed relative error. Ranking is what a misspecified model cannot do, and the simulator scores the executed plan honestly either way, so no model error can hide.

\paragraph{Open-loop iteration compounds model error.} With one input feature per output coordinate, the model is iterated open loop for five steps from held-out starts, feeding the logged actions so that only dynamics prediction is tested, and the final representation is decoded to a state and scored by $R^2$ against the state truly reached. On the unrotated representation the probe does worse than predicting the dataset mean, while the identical procedure on the rotated representation recovers a usable five-step model from the same budget of $64$ logged transitions.

\paragraph{A misspecified monitor hides violations inside its own error.} A held-out transition is corrupted by replacing one factor of the true next state with that factor's value from a different transition, chosen so that the jump is large and physically implausible. Clean and corrupted transitions are scored by the sparse model's prediction error, and the probe reports the separation as AUROC. Orthogonal maps preserve norms, so the corruption is exactly as large in the mixed representation as in the recovered one; what differs is the monitor's baseline, since a well-specified sparse model predicts clean transitions tightly while a misspecified one carries a large and variable clean error inside which the same violation can hide and slip past the detector without raising surprise.

\paragraph{No probe uses oracle information.} None of the probes needs the permutation or the signs that Theorem~\ref{thm:idjepa-zheng} leaves free: the editing probes select their coordinate from the start-to-goal difference and the model-based probes select their features by correlation, both name-free procedures, so every probe is invariant to exactly the ambiguity the theory does not resolve. All thresholds and budgets are fixed once and shared by both representations, with the same data, model class, and planner; the fitted rotation is the only difference between the two series in each panel of Figure~\ref{fig:sparse-use-summary}.

\subsection{Learned visual encoders}
\label{sec:app-learned-details}

The experiments of Section~\ref{sec:exp-visual-encoders} replace the analytic orbit with encoders trained from pixels. This section details those encoders, their probes, and the protocol behind them.

\paragraph{Setup.} Both visual suites render $N=8$ independent Gaussian factors at $64\times64$: in the visual-factor suite each factor sets the intensity of one localized colored part of the image, and in the object-scene suite each factor moves one colored object along its own lane. Consecutive frames follow the stationary transition of Section~\ref{sec:preliminaries}. The encoder is a small strided convolutional network with a linear head, trained with the alignment objective and LeJEPA's Gaussianity regularizer; nothing in the architecture or the objective is specific to DSReg. The supervised Procrustes oracle reported alongside solves the orthogonal Procrustes problem from the whitened representation onto the ground-truth latents, so it uses labels and marks the ceiling available to any method that can only rotate. The dependency rotation itself is fit exactly as on the analytic benchmarks, from pooled-pixel features and local ridge Jacobians. Labels enter only through benchmark construction and the evaluation metrics.

\paragraph{The gain survives learned encoders.} Learning replaces the exact orbit with an imperfect one, and these experiments ask whether the rotation's gain survives that replacement. The protocol guards the answer against overfitting. The rotation is fit on one split of estimated latents and fixed pooled-pixel features, and every recovery and sparse-use score in Table~\ref{tab:trained-visual-scaling} is computed on the held-out split. The ten-run aggregate shows the same signature as the analytic benchmarks, with dense $R^2$ unchanged while latent MCC, few-shot single-latent readout, and sparse-use scores all improve (Figure~\ref{fig:trained-visual-encoder-metrics}). Moreover, the probe families of Section~\ref{sec:exp-sparse-use} carry over to the same pixel-trained encoders, with control error falling, rollout $R^2$ rising, surprise detection improving, and DSReg ahead in every run of all three probes from pixels (Figure~\ref{fig:trained-visual-encoder-metrics}). The adaptation to pixels is minimal: transitions come from actuated episodes whose actions are recorded, the budgets and feature caps of Section~\ref{sec:app-sparse-use-details} are kept, control is scored by relative endpoint error under the benchmark's mean transition, and the surprise violations are rendered and re-encoded rather than transformed analytically, so the monitor sees the corruption only through the encoder, exactly as it would in deployment.

\begin{table}[tbp]
\centering
\caption{\textbf{Split-audited learned visual encoders.} The analytic signature survives learning: dense recovery is unchanged while individual-latent recovery, few-shot readout, and sparse use improve, with every score computed on a held-out split (the rotation is fit on a separate split of estimated latents and fixed pooled-pixel features). Arrows are ten-run means for LeJEPA $\to$ DSReg; sparse error is lower better.}
\label{tab:trained-visual-scaling}
\setlength{\tabcolsep}{4.5pt}
\begin{tabular}{@{}llcccc@{}}
\toprule
Benchmark & $N$ & Dense $R^2$ & MCC & Few-shot $R^2$ & Sparse error $\downarrow$ \\
\midrule
Visual factors & 8 & $0.918\pm0.002$ & $0.625\to0.958$ & $-0.081\to0.882$ & $2.13\to1.35$ \\
Object scenes & 8 & $0.924\pm0.002$ & $0.591\to0.961$ & $-0.143\to0.899$ & $2.29\to1.41$ \\
\bottomrule
\end{tabular}
\end{table}

\begin{figure}[H]
    \centering
    \includegraphics[width=0.86\linewidth]{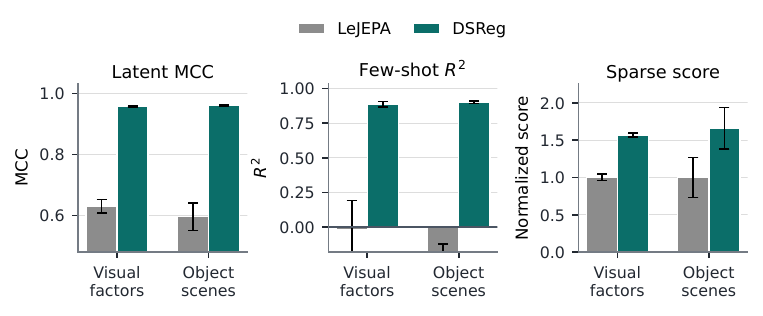}
    \caption{\textbf{The analytic signature survives pixel training.} The gain concentrates where individual latents matter: component-wise recovery and single-latent readout improve sharply, and sparse use follows. A negative few-shot value underperforms the held-out mean; the sparse-use score is normalized so that LeJEPA sits at one. On the pixel probes, sparse control error drops from $0.63$ to $0.52$, five-step sparse rollout $R^2$ rises from $0.74$ to $0.81$, and transition-surprise AUROC from $0.995$ to $0.999$ (LeJEPA\,/\,DSReg, ten runs).}
    \label{fig:trained-visual-encoder-metrics}
\end{figure}

\paragraph{Width mismatch costs nothing.} Figure~\ref{fig:learned-width} in the main text removes the one remaining piece of oracle knowledge in the encoder design, the assumption that representation width matches the latent dimension. With width $M\in\{16,32\}$ at true $N=8$, the top-$N$ right-singular subspace of the estimated dependency Jacobians selects the directions the observed features depend on. It is worth noting that variance could not make this selection, since every direction has unit variance under the Gaussian constraint. Within that subspace the dense span is recovered and DSReg again matches the supervised oracle to within $0.001$ at every width: the dependency structure finds the latent subspace inside a larger representation on its own, with no other use of the true dimension.

\paragraph{Training imbalance is tolerated.} Real training data is rarely balanced, so Figure~\ref{fig:distribution-stress} stress-tests the same protocol under imbalanced training distributions: long-tail settings make some factors rare during training, correlated settings make pairs of factors co-vary, and evaluation returns to a balanced held-out distribution in both cases, with the rotation fit on an unlabeled dependency split from the same training distribution (five runs). Dense $R^2$ stays between $0.80$ and $0.93$, unchanged by the rotation, and DSReg continues to improve individual-latent recovery in every setting, so the rotation does not depend on the training distribution matching the isotropic ideal of the theory, only on the dependency structure itself, which imbalance leaves intact by construction.

\begin{figure}[H]
    \centering
    \includegraphics[width=0.62\linewidth]{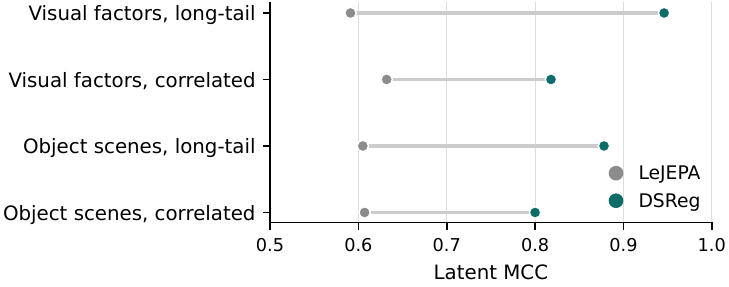}
    \caption{\textbf{Training imbalance does not break the rotation.} Long-tail and correlated training distributions leave the recovery gain intact, with the dense span fully preserved throughout.}
    \label{fig:distribution-stress}
\end{figure}

\FloatBarrier
\subsection{Gaussian 3DShapes}
\label{sec:app-shapes3d}

The first external-renderer benchmark tests DSReg on images whose generator we did not build, so the footprint structure the criterion reads is entirely the renderer's own.

\paragraph{The recovery carries over to an external renderer.} The 3DShapes benchmark uses the DeepMind renderer \citep{kim2018disentangling}: $z\sim \mathcal N(0,I_6)$ with OU pairing ($\rho=0.95$), each latent CDF-mapped onto one factor grid (floor hue, wall hue, object hue, scale, shape, orientation), and the pre-rendered RGB image as the observation, so the latent-to-pixel footprint structure is not of our design. Table~\ref{tab:shapes3d-gaussian} reports the aggregate scores; its caption explains why the low dense $R^2$ of the external $\beta$-VAE and $\beta$-TCVAE references measures axis-alignment rather than lost factor information.

\paragraph{The dependency criterion recovers the factors that move few pixels.} Figure~\ref{fig:shapes3d-perfactor} splits the scores per factor. The three hue factors paint large pixel regions and every method captures them to some degree, while scale, shape, and orientation move few pixels, and there the VAE baselines collapse outright. Reconstruction-driven objectives weight factors by the pixels they explain; the dependency criterion asks only which features a factor touches, so pixel-poor and pixel-rich factors fare alike.

\begin{table}[H]
\centering
\caption{\textbf{DSReg improves latent MCC on every run of Gaussian 3DShapes while matching LeJEPA's dense recovery.} Scores are latent MCC, DCI, MIG, SAP, and informativeness (test $R^2$ of a boosted readout), twenty runs, on an external renderer. The LeJEPA row is the same frozen representation DSReg rotates, which is why the two share dense $R^2$; the $\beta$-VAE/$\beta$-TCVAE rows are external references with the matched trunk, trained on i.i.d.\ images at an untuned $\beta{=}4$; their low dense $R^2$ reflects a linear readout only, with nonlinear informativeness staying high, so the gap measures axis alignment rather than lost information.}
\label{tab:shapes3d-gaussian}
\small
\setlength{\tabcolsep}{3pt}
\begin{tabular}{@{}lcccccc@{}}
\toprule
Method & Latent MCC & DCI & MIG & SAP & Info.\ $R^2$ & Dense $R^2$ \\
\midrule
LeJEPA & $0.640\pm0.041$ & $0.298\pm0.044$ & $0.069\pm0.026$ & $0.212\pm0.074$ & $0.854\pm0.003$ & $0.861\pm0.002$ \\
$\beta$-VAE & $0.272\pm0.025$ & $0.318\pm0.086$ & $0.080\pm0.038$ & $0.102\pm0.030$ & $0.738\pm0.076$ & $0.191\pm0.010$ \\
$\beta$-TCVAE & $0.277\pm0.043$ & $0.345\pm0.117$ & $0.102\pm0.067$ & $0.109\pm0.052$ & $0.751\pm0.087$ & $0.199\pm0.027$ \\
DSReg & $0.921\pm0.005$ & $0.901\pm0.017$ & $0.419\pm0.013$ & $0.842\pm0.016$ & $0.868\pm0.003$ & $0.861\pm0.002$ \\
\bottomrule
\end{tabular}
\end{table}

\begin{figure}[tbp]
    \centering
    \includegraphics[width=0.8\linewidth]{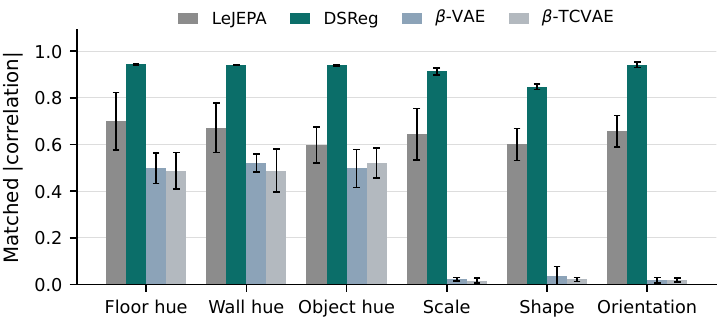}
    \caption{\textbf{Per-factor recovery on Gaussian 3DShapes.} DSReg improves every factor and alone recovers scale, shape, and orientation, where the VAE baselines collapse almost entirely.}
    \label{fig:shapes3d-perfactor}
\end{figure}

\paragraph{The recovery is visible without a decoder.} Sweeping one factor with the others held fixed selects real rendered images from the exhaustive factor grid, and the learned latents are evaluated on exactly those frames; factor labels enter only in this evaluation and never during learning. In Figure~\ref{fig:shapes3d-interventions} each sweep moves a single rotated latent while the others stay flat and the unrotated LeJEPA coordinates respond weakly and jointly. The matched correlation matrices of Figure~\ref{fig:shapes3d-corr} carry the same contrast in aggregate over the evaluation set, with no decoder anywhere in the loop.

\begin{figure}[tbp]
    \centering
    \includegraphics[width=\linewidth]{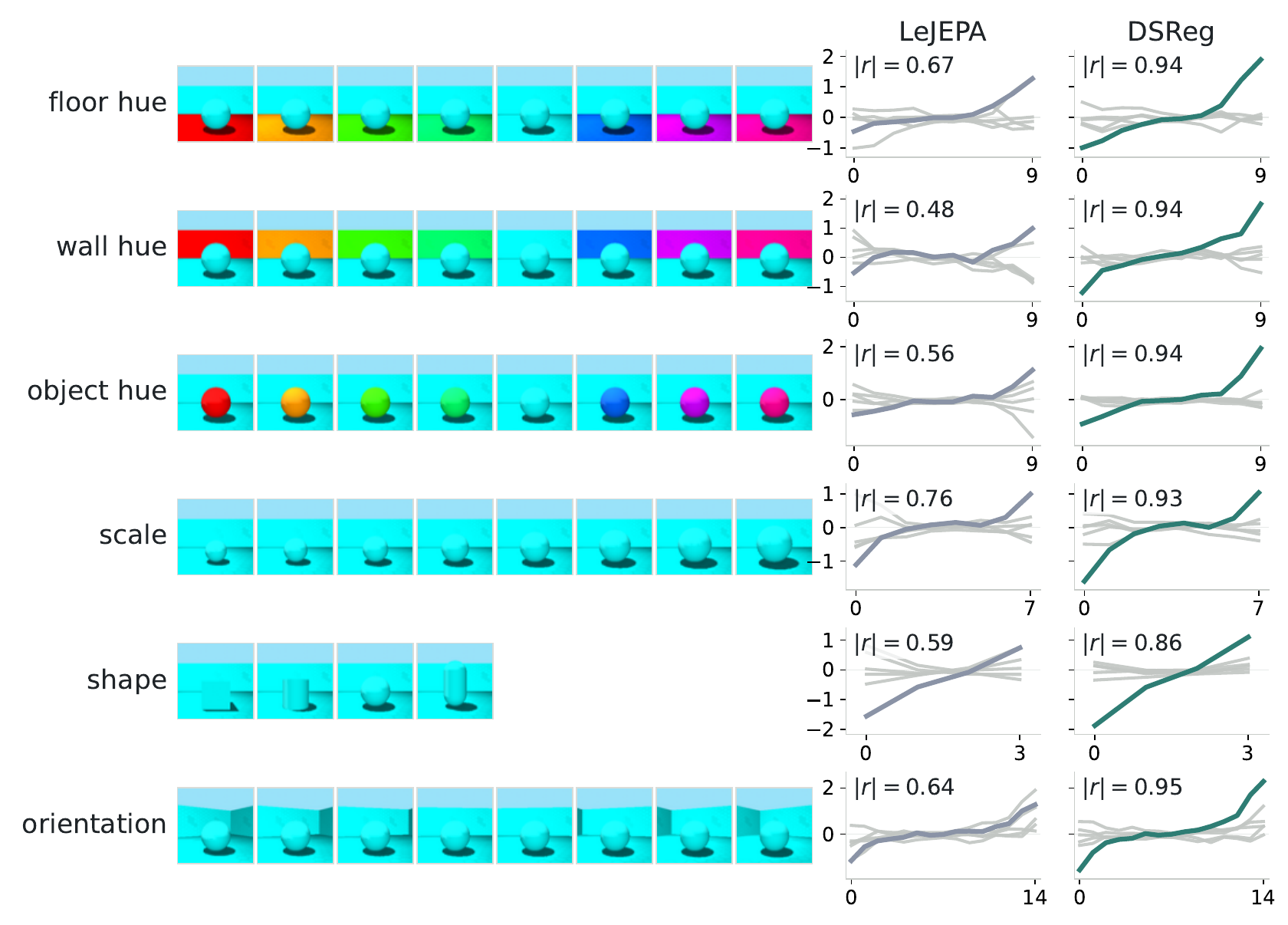}
    \caption{\textbf{The rotation turns entangled responses into one selective response per factor.} One run of Gaussian 3DShapes. Each row sweeps one factor with all others fixed (rendered frames on the left); the right panels show every learned latent's standardized response in LeJEPA and in DSReg coordinates, on a shared scale per row. The flat gray curves carry the no-mixing claim across all six factors.}
    \label{fig:shapes3d-interventions}
\end{figure}

\begin{figure}[tbp]
    \centering
    \includegraphics[width=0.62\linewidth]{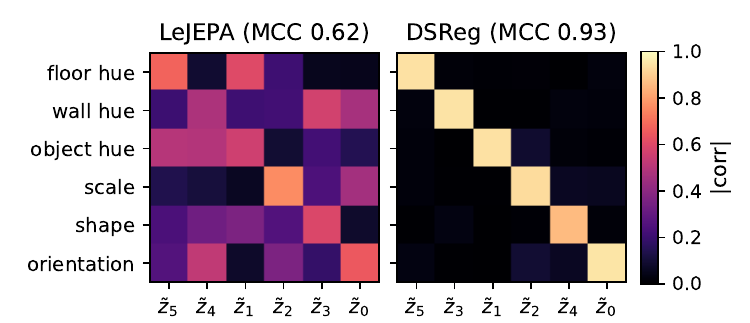}
    \caption{\textbf{Matched correlation heatmaps.} For the run of Figure~\ref{fig:shapes3d-interventions}, columns are permuted by each method's own best assignment. The same encoder yields a smeared matrix in the unrotated coordinates and a near-diagonal one after rotation, in agreement with the per-factor responses of the figure above.}
    \label{fig:shapes3d-corr}
\end{figure}

\FloatBarrier
\subsection{Quarter-orientation dSprites}
\label{sec:app-dsprites}

The second external renderer is harder by construction: a symmetry of the sprites breaks the linear premise unless the factor grid is restricted, making it a test under an imperfect linear stage.

\paragraph{Setup.} The dSprites benchmark \citep{matthey2017dsprites} follows the same protocol as 3DShapes: $z\sim\mathcal N(0,I_5)$ with OU pairing ($\rho=0.95$), each latent CDF-mapped onto one factor grid (shape, scale, orientation, position~$x$, position~$y$), and the rendered image used as the observation. Orientation is restricted to a quarter turn because the sprites are rotationally symmetric, so orientation over a full turn is not a function of the image; without the restriction the linear-identifiability premise fails for every method. What remains is a benchmark on which the premise is only partially attained, a stress test of the rotation under an imperfect linear stage rather than another clean win. The dense $R^2$ is also the binding ceiling: it is invariant to any rotation, so no selection criterion can recover more of the factors than the encoder linearly contains, and the scores should be read against that ceiling.

\paragraph{The premise, once attained, is the only bottleneck.} Figure~\ref{fig:dsprites-quarter} shows the outcome over twenty runs, with each VAE at its best $\beta$ from the sweep $\{1,2,4,8,16\}$ and twenty runs per $\beta$. DSReg improves over LeJEPA and exceeds both VAE baselines (one-sided Mann--Whitney $p=8\times10^{-5}$ and $p=4\times10^{-5}$), which fits the pattern seen throughout the paper, where the premise rather than the rotation step is what binds recovery on every benchmark in the paper.


\begin{figure}[H]
    \centering
    \includegraphics[width=0.46\linewidth]{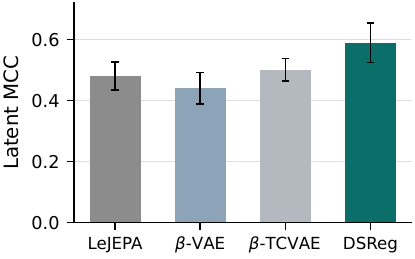}
    \caption{\textbf{DSReg beats both VAE baselines under a partially attained premise.} Quarter-orientation dSprites, twenty runs, dense $R^2\approx0.53$ as the rotation-invariant ceiling; each VAE at its best $\beta$. One-sided Mann--Whitney tests are against $\beta$-VAE and $\beta$-TCVAE at their own best $\beta$; higher MCC is better.}
    \label{fig:dsprites-quarter}
\end{figure}

\FloatBarrier
\subsection{Rotation baselines and dense-use sanity checks}
\label{sec:app-baselines}

The last question is whether anything simpler could have done the work. The baselines get two chances: rotations that never see observations, and a family reading the same signal.

\paragraph{Latent-only rotations cannot break the rotational symmetry.} The latents of the exact analytic orbit are isotropic Gaussian, so PCA, Varimax, and FastICA face the very symmetry that motivates DSReg and stay at the unrotated level alongside LeJEPA and a random orbit, while DSReg from exact Jacobians approaches the oracle and drives the dependency objective to the oracle level (Table~\ref{tab:rotation-baselines}). The residual gap reflects the $\ell_1$ relaxation rather than estimation error: on a fraction of random worlds the relaxation's global minimizer is a slightly mixed rotation even though the support criterion of Theorem~\ref{thm:idjepa-zheng} still identifies the permutation. Each baseline, with Varimax and FastICA applied after whitening, returns an orthogonal transformation of the same estimated latents DSReg receives.

\begin{table}[tbp]
\centering
\caption{\textbf{Latent-only rotation baselines.} The observation-side signal rather than the span is what does the work: all methods start from the same linearly identifiable representation, but PCA, Varimax, and FastICA see only estimated-latent samples and cannot break the rotational symmetry, while the dependency criterion $\partial x/\partial(Rh)$ can, because it also sees how the observed variables respond to the applied rotation.}
\label{tab:rotation-baselines}
\fitwidth{
\begin{tabular}{lccc}
\toprule
Method & Latent MCC & $R^2(h\to z)$ & Dependency sparsity $\downarrow$ \\
\midrule
LeJEPA & $0.65\pm0.03$ & $1$ & $0.32\pm0.02$ \\
Random orbit & $0.64\pm0.03$ & $1$ & -- \\
PCA & $0.62\pm0.04$ & $1$ & -- \\
Varimax & $0.64\pm0.03$ & $1$ & -- \\
FastICA & $0.64\pm0.05$ & $1$ & -- \\
DSReg & $0.90\pm0.05$ & $1$ & $0.18\pm0.02$ \\
Oracle & $1$ & $1$ & $0.18\pm0.02$ \\
\bottomrule
\end{tabular}}
\end{table}

\paragraph{Freed from the symmetry, the baselines still leave the latents mixed.} On the learned, non-isotropic representations the baselines have a real opening, yet Figure~\ref{fig:learned-baselines} in the main text shows the latents still mixed while DSReg nearly matches the supervised Procrustes oracle of Section~\ref{sec:app-learned-details}. Whatever asymmetry learning leaves in the representation does not point toward the world's variables; the observation-side dependency signal does, on both learned visual suites.

\paragraph{Dense readouts and retrieval do not move.} Table~\ref{tab:dense-sparse-sanity} checks the other side: dense readouts and retrieval are unchanged to two decimals on every task, so everything DSReg gains over these baselines, it gains in sparse access rather than in the content the readouts measure.

\begin{table}[tbp]
\centering
\caption{\textbf{Dense versus sparse-use sanity check.} Everything the rotation gains, it gains in sparse access: dense readout and ordinary retrieval are unchanged, with dense $R^2$ at $1/1$ (LeJEPA\,/\,DSReg) for every task, while sparse-control error drops whenever only a few estimated latents may be used.}
\label{tab:dense-sparse-sanity}
{\small
\setlength{\tabcolsep}{4pt}
\begin{tabular}{lcc}
\toprule
Task & Retrieval cost LeJEPA / DSReg & Sparse err.\ LeJEPA / DSReg \\
\midrule
TwoRoom & $1.31 / 1.31$ & $0.81 / 0.58$ \\
PushT & $1.45 / 1.45$ & $1.45 / 1.10$ \\
FetchSlide & $1.43 / 1.43$ & $1.24 / 0.85$ \\
\bottomrule
\end{tabular}}
\end{table}

\paragraph{Access to the dependency signal alone does not close the gap.} A stronger baseline consumes the signal itself. On ten fresh runs of each renderer benchmark we optimize the scale-invariant column-support sparse-ICA criterion of \citet{ng2023sparseica} on the same estimated anchor Jacobians that DSReg consumes, with the same orthogonal parameterization and restarts, and it recovers part of the structure while DSReg stays ahead on both benchmarks (Table~\ref{tab:sparse-ica-baseline}). The advantage lies in the criterion itself rather than in mere access to the signal that both criteria consume.

\begin{table}[H]
\centering
\small \caption{\textbf{Sparse ICA on the same dependency signal recovers less.} Latent MCC, mean~$\pm$~std over ten fresh runs per renderer benchmark; both criteria consume the same estimated anchor Jacobians, orthogonal parameterization, and restarts. Both rows share every setting except the criterion.}
\label{tab:sparse-ica-baseline}
\begin{tabular}{lcc}
\toprule
Criterion & Gaussian 3DShapes & Quarter-orientation dSprites \\
\midrule
Column-support sparse ICA & $0.76\pm0.06$ & $0.57\pm0.05$ \\
DSReg & $0.92\pm0.01$ & $0.60\pm0.08$ \\
\bottomrule
\end{tabular}
\end{table}

\FloatBarrier

\end{document}